\pdfoutput=1

\documentclass[11pt]{article}

\usepackage[preprint]{acl}

\usepackage{amsmath,amsfonts,bm}
\usepackage[utf8]{inputenc} 
\usepackage{amsthm}
\usepackage{amssymb}
\usepackage{mathtools} 

\newcommand{\cC}{{\cal C}}
\newcommand{\cQ}{{\cal Q}}
\newcommand{\cG}{{\cal G}}

\def\eqref#1{equation~\ref{#1}}

\def\1{\bm{1}}

\DeclareMathAlphabet{\mathsfit}{\encodingdefault}{\sfdefault}{m}{sl}
\SetMathAlphabet{\mathsfit}{bold}{\encodingdefault}{\sfdefault}{bx}{n}

\newcommand{\R}{\mathbb{R}}

\usepackage{subcaption}
\usepackage{booktabs}
\usepackage[noabbrev]{cleveref}
\usepackage{thmtools}

\newcommand{\N}{\mathbb{N}} 

\newcommand{\Exp}[1]{{\rm E} \left[ #1 \right]}  
\newcommand{\eqdef}{\coloneqq}

\usepackage[colorinlistoftodos,bordercolor=orange,backgroundcolor=orange!20,linecolor=orange,textsize=scriptsize]{todonotes}

\theoremstyle{plain}
\newtheorem{assumption}{Assumption}
\newtheorem{theorem}{Theorem}
\numberwithin{lemma}{section}
\theoremstyle{definition}
\numberwithin{definition}{section}

\usepackage{algorithm}
\usepackage{algpseudocode}

\usepackage{diagbox}
\usepackage{multirow}

\usepackage{times}
\usepackage{latexsym}

\usepackage[T1]{fontenc}

\usepackage[utf8]{inputenc}

\usepackage{microtype}

\usepackage{xcolor}

\usepackage{inconsolata}

\usepackage{graphicx}

\renewcommand{\paragraph}[1]{\vspace{0.1em}\noindent\textbf{#1}}

\newcommand{\idlow}[1]{\mathord{\mathcode`\-="702D\it #1\mathcode`\-="2200}}
\newcommand{\id}[1]{\ensuremath{\idlow{#1}}}

\newcommand{\bits}{b}
\newcommand{\round}{\operatorname{rnd}}
\newcommand{\zeros}{z}
\newcommand{\scale}{s}

\title{Pushing the Limits of Large Language Model Quantization \\ via the Linearity Theorem}

\author{%
  Vladimir Malinovskii$^\dagger$ \\
  Yandex, HSE University \\
  \And
  Andrei Panferov$^\dagger$ \\
  ISTA$^\triangledown$ \\
  \And
  Ivan Ilin \\
  GenAI CoE, KAUST$^*$ \\
  \AND
  Han Guo \\
  MIT$^\diamond$ \\
  \And
  Peter Richtárik \\
  GenAI CoE, KAUST$^*$ \\
  \And
  Dan Alistarh \\
  ISTA$^\triangledown$, NeuralMagic \\
}

\begin{document}

\maketitle

\def\thefootnote{$\dagger$}\footnotetext{Equal contribution}
\def\thefootnote{$\diamond$}\footnotetext{ Massachusetts Institute of Technology} \def\thefootnote{$\triangledown$}\footnotetext{Institute of Science and Technology Austria} \def\thefootnote{\arabic{footnote}}
\def\thefootnote{*}\footnotetext{
King Abdullah University of Science and Technology, Saudi Arabia}\def\thefootnote{\arabic{footnote}}

\begin{abstract}
Quantizing large language models has become a standard way to reduce their memory and computational costs. Typically, existing methods  focus on breaking down the problem into individual layer-wise sub-problems, and minimizing per-layer error, measured via various metrics. Yet,  this approach currently lacks theoretical justification and the metrics employed may be sub-optimal. 
In this paper, we present a ``linearity theorem'' establishing a direct relationship between the layer-wise $\ell_2$ reconstruction error and the model perplexity increase due to quantization. This insight enables two novel applications: (1) a simple \emph{data-free LLM quantization method using Hadamard rotations and MSE-optimal grids}, dubbed HIGGS, which outperforms all prior data-free approaches such as the extremely popular NF4 quantized format, and 
(2) an \emph{optimal} solution to the problem of finding non-uniform per-layer quantization levels which match a given compression constraint in the medium-bitwidth regime, obtained by reduction to dynamic programming. On the practical side, we demonstrate improved accuracy-compression trade-offs on Llama-3.1 and 3.2-family models, as well as on Qwen-family models. Further, we show that our method can be  efficiently supported in terms of GPU kernels at various batch sizes, 
advancing both data-free and non-uniform quantization for LLMs.
\end{abstract}

\section{Introduction}

Quantization has become a standard technique for reducing the memory costs of large language models (LLMs), e.g.~\citep{dettmers2022llm, dettmers2022case, frantar2022gptq, lin2023awq, chee2023quip, tseng2024quipbetterllmquantization, van2024gptvq}. 
Most existing high-performance approaches start from the natural strategy of quantizing layers one-at-a-time, while minimizing a given per-layer error function, such as per-layer quantization entropy~\citep{dettmers2023qloraefficientfinetuningquantized} or $\ell_1$-norm error~\citep{yoshida2023nf4isntinformationtheoretically}. 

In this context, a reasonable question regards the relationship between \emph{the individual per-layer quantization error}, measured in terms of, e.g., output MSE, and \emph{the model's output error}, measured in terms of, e.g., validation perplexity (PPL). While previous work observes that various layers can have vastly different ``sensitivities'' towards the model's output~\citep{owl, frantar2022spdy}, it is currently not clear how these can be estimated. Moreover, it is not clear what the correct error metric quantization techniques should be minimizing. 

\paragraph{Contribution.} In this paper, we start by examining  the relationship between per-layer error and global error from the theoretical perspective, and identify a natural ``linearity theorem'', which precisely links the model's performance in terms of output loss (or PPL), with the per-layer MSE quantization error over the weights. Roughly speaking, we show that, for reasonable quantization bit-widths, the relationship between per-layer MSE and output PPL is \emph{linear}, modulo a \textit{constant} scaling coefficient per layer, which is \textit{independent of the quantization approach}. 
We show experimentally that the theorem works remarkably well at predicting quantization error for different schemes in the 3-8 bit range; see Figure~\ref{fig:perplexity-validation} for an illustration.

The linearity theorem, whose technical prerequisites and complete proof we provide, has a few non-trivial practical implications. 
First, it guides us towards a state-of-the-art \emph{data-free} quantization method. 
Specifically, we start by observing that, for a fixed per-layer compression budget, the linearity theorem implies that minimizing the perplexity increase can be reduced to minimizing the individual, per-layer MSE quantization errors. Moreover, we observe that we can do this in a \emph{calibration-free} fashion if the model weights are incoherence-processed via Hadamard rotations---known to make them approximately Gaussian---and then we quantize using Gaussian-MSE-optimal grids, which are efficiently computable~\citep{GaussianCase}.  The resulting method, called HIGGS (for Hadamard Incoherence with Gaussian MSE-optimal GridS) is highly accurate and efficiently implementable for various bit-widths and grid constraints.

\begin{figure}[t]
    \centering
    \includegraphics[width=1\linewidth]{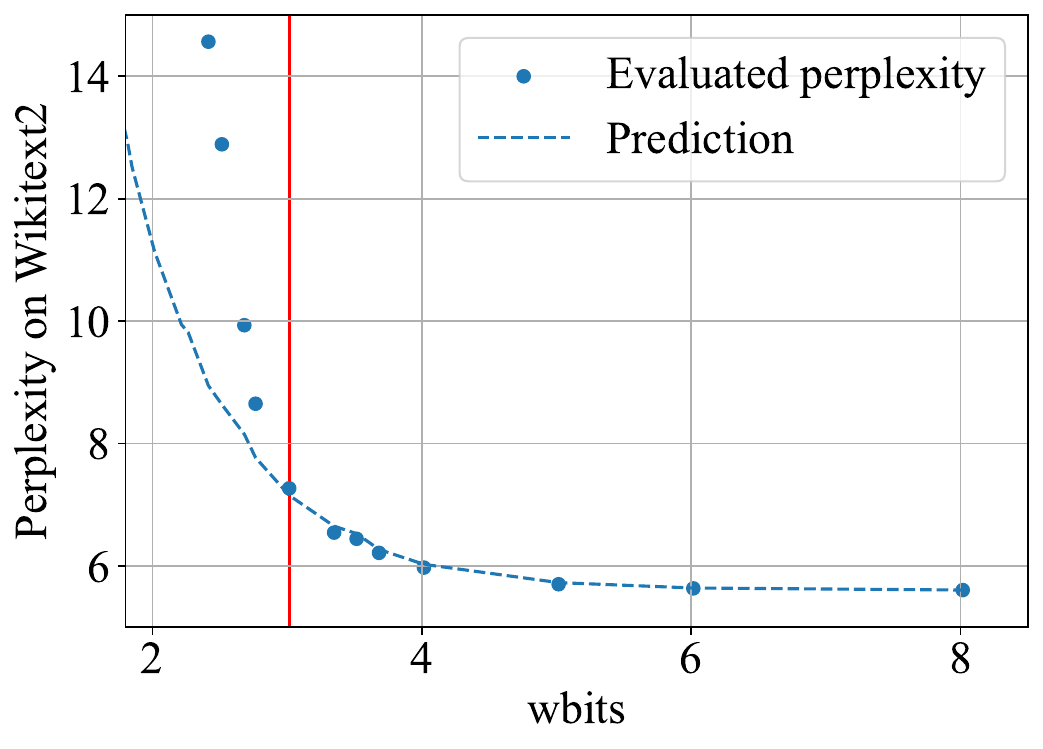}
    \caption{Actual \textit{measured} Perplexity (PPL) of quantized models versus \textit{predicted} PPL, following Theorem~\ref{main:ppl-theorem}, for uniform HIGGS quantization of Llama 3.1 8B in the 2--8 bit range. Vertical line shows the limit of the theorem's applicability}
    \label{fig:perplexity-validation}
\end{figure}

The second practical application of the linearity theorem comes for solving the \emph{non-uniform quantization} problem: 
that is, the problem of finding the per-layer bit-widths which satisfy a fixed constraint on the total model size / average bits per parameter, which minimize the perplexity increase. In the range of applicability of the linearity theorem, we show that optimal non-uniform compression can be reduced to knapsack-style dynamic programming over the set of quantization choices at each layer. Interestingly, in this range, this problem can be solved \emph{optimally} using existing linear programming solvers; in practice, solving an LLM-sized instance can be done in seconds.    
While this procedure requires computation of the per-layer linear scaling coefficients, we show that this can be done efficiently and even \emph{data-free}, based   on randomly sampled input token sequences. 
Moreover, interestingly, the two applications can be compounded, yielding an \emph{optimal non-uniform data-free} quantization technique, which we call \emph{dynamic HIGGS}. 

\begin{figure}[t]
    \centering
    \includegraphics[width=1\linewidth]{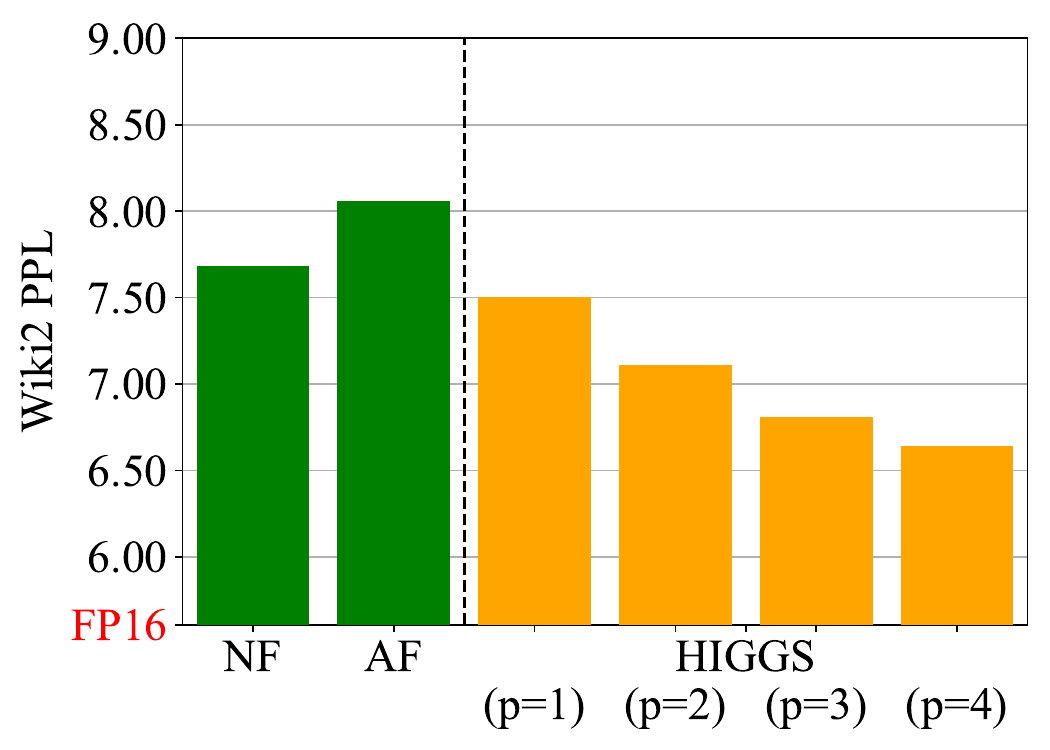}
    \caption{Comparison of Normal Float (NF), Abnormal Float (AF) and HIGGS on Llama 3.1 8B quantization to 3.19-3.25 bitwidth range. HIGGS is instantiated at different lattice dimensionalities $p$.}
    \label{fig:compact_ppl_comparison}
\end{figure}

We validate our practical applications experimentally by quantizing the popular Llama 3.1 and Llama 3.2 models~\citep{dubey2024llama3herdmodels}, as well as a Qwen model~\citep{bai2023qwen}, across a wide range of bit-widths, and evaluating on standard perplexity (PPL) and in-context learning (ICL) benchmarks. 
The results, sampled in Figure~\ref{fig:compact_ppl_comparison}, confirm the fact that HIGGS with uniform quantization can outperform the Normal Float (NF) and Abnormal Float (AF) formats for lower bit-widths in the 3-4 bit range, as well as the recent data-free HQQ method~\citep{badri2023hqq}.
At the same time, for higher bitwidths, we observe that all methods produce results in the same accuracy range. 
The \emph{dynamic, non-uniform} variant of HIGGS provides consistent additional accuracy boosts, and appears to lead state-of-the-art results for quantization methods with efficient hardware support. 
Surprisingly, we observe that dynamic HIGGS can even outperform \emph{calibration-based} methods such as GPTQ~\citep{frantar2022gptq} and AWQ~\citep{lin2023awq} in the 3-4 bit-width range. 
Moreover, HIGGS can be applied in conjunction with GPTQ, leading to state-of-the-art accuracy results for scalar quantization.  

On the runtime side, 
we show that our approach can be supported extremely efficiently via GPU kernels. 
Specifically we show that the recent FLUTE kernel design~\citep{guo2024fastmatrixmultiplicationslookup} can be adapted to support a subset HIGGS multi-dimensional grids, providing a high accuracy solution that is efficient across various batch sizes. Our solution can be integrated with both Pytorch~\citep{paszke2019pytorch} and vLLM~\citep{kwon2023efficient}, leading to speedups of 2-3x relative to FP16 precision, at a low decrease in accuracy relative to the FP16 baseline.

\section{Background and Related Work}

Post-training quantization~\citep{nagel2020up, gholami2021survey} of LLMs has become an extremely active research area. Here, we provide some background, focusing on the work closest to ours. 
The focus of early work in LLM quantization has been on \emph{data-free} methods~\citep{dettmers2022llm, yao2022zeroquant, park2022nuqmm} using direct round-to-nearest (RTN) quantization over small weight groups.  
For example, given a group of $g$ consecutive layer weights, viewed as a vector $\mathbf{x} \in \R^g$, we define $b$-bit RTN as
\begin{align}
\label{eq:quantization}
\cQ(\mathbf{x},\bits) &= \round \Bigg(\frac{\mathbf{x} - \min(\mathbf{x})}{\max(\mathbf{x}) - \min(\mathbf{x})} (2^\bits -1) \Bigg)  \nonumber \\
&= \round((\mathbf{x} - \zeros(\mathbf{x}))/\scale(\mathbf{x})),
\end{align}

\noindent where $\round$ rounds to the nearest integer level, $\zeros = \zeros(\mathbf{x}) = \min(\mathbf{x})$ is the ``zero point'' and $\scale = \scale(\mathbf{x}) =  (\max(\mathbf{x}) - \min(\mathbf{x})) / (2^\bits -1)$ is the min-max scale.

One key issue with this first wave of data-free RTN methods is that they tend to yield high accuracy loss below 8 bits per parameter. This can be addressed primarily in two ways: (1) by improving the rounding function in a \emph{data-aware} way, and (2) by using more complex \emph{non-uniform grids}.

\paragraph{Data-Aware Methods.} 
Calibration-based methods such as GPTQ~\citep{frantar2022gptq} improved significantly upon RTN by allowing a subset of weights to be adjusted during quantization, based on a sample of calibration data.  
Follow-up methods such as AWQ~\citep{lin2023awq},  SqueezeLLM~\citep{kim2023squeezellm}, OWQ~\citep{lee2024owq} and SpQR~\citep{dettmers2023spqr} implemented variants of outlier-aware quantization, where a small fraction of weights are effectively stored in higher precision.
Further, high-compression methods such as QuIP~\citep{chee2023quip}, QuIP\#~\cite{tseng2024quipbetterllmquantization}, QTIP~\citep{tseng2024qtip} and AQLM~\citep{egiazarian2024extreme} investigated much more complex quantized representations, such as lattice quantization, often paired with incoherence pre-processing of the weights, and GPTQ-like weight updates. 
While such methods can be Pareto-competitive down to 2 bits per parameter, some practical disadvantages are 1) the reliance on task-specific calibration data, 2) the relatively high processing time to produce models, as well as 3) the complexity of efficiently supporting  lattice representations at runtime. 

We emphasize the fact that the linearity theorem has no direct bearing on the \emph{data-aware} layer-wise MSE minimization problems considered in references such as GPTQ and QuIP, which are of the form $$\min_{\widehat{W}_l \in \Omega_l}\| W_l^{\star} X - \widehat{W}_l X\|_F^2,$$ where $W_l^{\star}$ is a matrix of pre-trained weights corresponding to layer $l$, $\widehat{W}_l=\cQ_l(W_l^{\star})$ represents the quantized weights, $X$ represents the layer's input, $
\Omega_l$ represents the collection of feasible/allowed quantized matrices, and $\|\cdot\|_F$ is the Frobenius norm. Here, we focus on the data-free case, and relate it to the quantization MSE over weights.

\paragraph{Data-free Non-Uniform Quantization.} 
To address the limitations of calibration-based methods, highly-popular open-source LLM inference frameworks such as \texttt{bitsandbytes}~\texttt{(BNB)}~\citep{bitsandbytes} employ \emph{data-free} quantization, but under optimized \emph{non-uniform} grids, designed to reduce reconstruction error. 
Specifically,~\citet{dettmers2023qloraefficientfinetuningquantized} proposed Normal Float (NF) grids which minimize \emph{quantization entropy}, while Abnormal Float (AF) \citep{yoshida2023nf4isntinformationtheoretically}, optimizes $\ell_1$ reconstruction error, arguing that it leads to better accuracy than NF. 
To optimize for those quantities, these works assume that LLM weights follow a \emph{zero-mean Gaussian distribution}, but do not enforce this assumption in any way. HQQ~\citep{badri2023hqq} provides and data-free algorithm to optimize the scale and zero-point for \emph{uniform} grids, while FLUTE~\citep{guo2024fastmatrixmultiplicationslookup}
 provides efficient GPU support for 1D non-uniform grids such as those of NF. 

 Recent work on data-aware methods~\citep{chee2023quip, tseng2024quipbetterllmquantization, ashkboos2024quarotoutlierfree4bitinference, liu2024spinquantllmquantizationlearned} applies  \emph{incoherence pre-processing} to the weights, often in the form of Hadamard transforms, to enforce a better match between the distribution of processed weights and the Gaussian. Yet, surprisingly, incoherence has so far only been used in the context of \emph{data-aware} and \emph{uniform-grid} quantization methods. 

Our work starts from a simple and general way of linking per-layer compression error with the global model loss increase. This inspires two different applications to data-free and non-uniform quantization, which are complementary to the aforementioned work. 

\paragraph{Additional Related Work.} 
The combination of Hadamard preprocessing and Gaussian MSE-optimal grids has also been proposed for \emph{gradient compression} in distributed optimization \citep{vargaftik2021driveonebitdistributedmean, vargaftik2022edencommunicationefficientrobustdistributed, davies2020new}. Gaussian MSE-optimal grids have been applied to multi-dimensional numerical integration by \citet{GaussianCase}. From them, we borrow the CLVQ algorithm for optimal grid computation given a set of parameters.

\section{The Linearity Theorem}

This section provides an overview of the linearity theorem, which links the layer-wise L2 error induced by quantization, to the increase in model perplexity, providing a theoretical foundation for weight quantization methods.

\subsection{Notation}

{\bf Pre-trained model.} Let $W^\star \eqdef (W_1^\star,\dots,W_L^\star)$, where for each $l$ in the set $\{1,\dots,L\}$, by $W_l^\star \in \R^{d^l_{in} \times d^l_{out}}$ we denote the matrix representing a linear layer of the pre-trained model  we are interested in compressing/quantizing. 

{\bf Reshaping operator.} Given a layer index $l$, let ${\cal R}_l: \R^{d_{in}^l \times d_{out}^l}  \to \R^{d_{in}^l \cdot d_{out}^l}$ be the ``reshaping'' operator, reshaping a matrix into a large-dimensional vector. That is, ${\bf w}_l ={\cal R}_l(W_l)$ is the vector obtained from the matrix $W_l$ by concatenating entries of $W_l$ into a single $d^l\eqdef d_{in}^l \times d_{out}^l$ dimensional vector. The entries can be concatenated in any order as long as it is always fixed.  Note that $\|W_l\|_F = \|{\cal R}_l(W_l)\|_2 = \|{\bf w}_l\|_2$. Further, let ${\cal R}_l^{-1}$ be the inverse reshaping operator mapping ${\bf w}_l$ back to $W_l$, such that
$ {\cal R}_l^{-1}({\bf w}_l) = {\cal R}_l^{-1}( {\cal R}_l(W_l ))  = W_l$.  Let ${\bf w} \eqdef ({\bf w}_1,\dots,{\bf w}_L ) \in \R^d$, where $d \eqdef \sum_{l=1}^L d^l$, and ${\cal R}^{-1}({\bf w}) \eqdef ({\cal R}_1^{-1}({\bf w}_1), \dots, {\cal R}_L^{-1}({\bf w}_L)) $.
Define ${\cal R}$ in a similar manner, and let ${\bf w}^\star \eqdef {\cal R}(W^\star)$ be the ``vector'' representation of the pre-trained model in $\R^d$.

{\bf Perplexity.} Let $\phi:\R^{d} \to \R$ be the perplexity function on $\R^d$ defined formally as $$\phi({\bf w}) \eqdef PPL({\cal R}^{-1}({\bf w})),$$
where $PPL$ is the perplexity function operating in the space of $W$.

\subsection{Technical Assumptions}

 Our results hold under the following assumptions, which we describe and discuss below. 

\begin{assumption}[Local optimality of the pre-trained weights] \label{ass:min} The \emph{uncompressed model weights} ${\bf w}^\star$  are a local minimizer of perplexity $\phi$.
\end{assumption}

 It is easy to see that if $W^\star$ is a local minimizer of $PPL$ if and only if ${\bf w}^\star = {\cal R}(W^\star)$ is a local minimizer of $\phi$. We emphasize that Assumption~\ref{ass:min} is not needed if the compression  mechanism used to compress each matrix $W^\star_l$ is unbiased, i.e., if $\Exp{\widehat{W}_l} = W^\star_l$ for all $ l\in \{1,\dots,L\}.$ We will leverage this observation in Section~\ref{sec:error-prediction}.



\begin{assumption}[Local smoothness of perplexity] \label{ass:smooth} The perplexity function $\phi$ is three times continuously differentiable in a neighborhood of ${\bf w}^\star$.
\end{assumption}

Recall that  $W^\star=(W_1^\star,\dots,W_L^\star)$ represents the pre-trained weights.  Let 
 $D^\star_l \eqdef \|W^\star_l\|_F I_{d_l} \in \R^{d_l\times d_l}$ for $l\in \{1,\dots,L\}$, and $$D^\star \eqdef Diag(D^\star_1,\dots,D^\star_L)\in \R^{d \times d},$$ 
 where $I_{d_l}$ is the $d_l\times d_l$ indentity matrix.
 
\begin{assumption}[Regularity of pre-trained weights] \label{ass:reg} There  exists a block-diagonal matrix $Z=Diag(Z_1,\dots,Z_L)$, where $Z_l = z_l I_{d_l}$,  $z_l >0$  for all $l \in \{1,\dots,L\}$, such that
 \begin{equation}
    \label{eq:product_ass3}
     D^\star \nabla^2 \phi\left( {\cal R}(W^\star)\right) D^\star \approx Z.
 \end{equation}
\end{assumption}

\paragraph{Discussion.}
If $\phi$ is twice differentiable (which is implied by Assumption~\ref{ass:smooth}), and ${\bf w}^\star$ is a local minimizer of $\phi$ (see Assumption~\ref{ass:min}), then the Hessian $\nabla^2 \phi({\bf w}^\star)$ is necessarily positive semi-definite. Clearly, $D^\star$ is diagonal with non-negative entries, and hence it is positive semi-definite. Therefore, $D^\star \nabla^2 \phi\left( {\cal R}(W^\star)\right) D^\star$ is also positive semi-definite. Let $\lambda_{\min}$ (resp.\ $\lambda_{\max}$) be the smallest (resp.\ the largest) eigenvalue of $D^\star \nabla^2 \phi\left( {\cal R}(W^\star)\right) D^\star$. Then
\[ \lambda_{\min} I_d \preceq  D^\star \nabla^2 \phi\left( {\cal R}(W^\star)\right) D^\star \preceq \lambda_{\max} I_d .\]
 We validate that Assumption~\ref{ass:reg} holds on language models in Appendix~\ref{sec:assmp_3_justification}.

\subsection{Theorem Statement} 

With this in place, we can now state our result. 

\begin{theorem}[Linearity theorem]\label{main:ppl-theorem}
Let the above assumptions hold.
Given an arbitrary layer index $l$ and an arbitrary (possibly stochastic) quantizer function $\cQ_l$,  let $\widehat{W}_l \eqdef \cQ_l({W}_l^\star)$ be the compressed version of the layer weights $W_l^\star$.
For each layer $l$, define the parameter $t_l$ as the relative quantization error, i.e.: 
\begin{equation}t_l^2 = t_l^2(W_l,\cQ_l)\eqdef \frac{\Exp{ \| \widehat{W}_l - W_l^\star \|_F^2 } }{  \|W_l^\star\|_F^2}, \label{eq:t_l-def} \end{equation}
Then, as long as $t_1,\dots,t_L$ are small enough, the following linear approximation of the expected perplexity holds:
\begin{eqnarray}\Exp{\id{PPL}(\widehat{W})}   \approx  \id{PPL}(W^\star)  +  \sum_{l=1}^L \alpha_l t_l^2, \label{eq:linearity}
\end{eqnarray}
where expectation is taken w.r.t.\ the randomness in the compression process, and the terms $\alpha_l$ are layer specific constants that are independent of the compression process.
\end{theorem}

\paragraph{Discussion.} 
The proof of the above result can be found in Appendix~\ref{app:ppl-theorem}.
The result essentially says that, given an \emph{arbitrary} (possibly randomized) perturbation function applied over the weights, if we can compute bounds $t_l$ on the (relative) Frobenius norm of the perturbation at each layer, then there exist layer-wise constant coefficients $\alpha_l$ such that the linear approximation of the global perplexity increase in Eqn.~(\ref{eq:linearity}) holds.
Importantly, the coefficients $\alpha_l$ are ``universal,'' in the sense that their values depend only on the layer weights, and not on the quantization function.  
In the following, we will explore two of its practical implications.

\section{HIGGS: Hadamard Incoherence and Gaussian MSE-Optimal Grids}

\subsection{The Hadamard Incoherence Trick}\label{sec:linear_hadamard}

Theorem~\ref{main:ppl-theorem} defines two sets of coefficients for each layer $l$: (1) The \emph{error coefficients} $t_l$, which measure the error relative to the layer's norm; and (2) The \emph{scaling coefficients} $\alpha_l$ measuring the importance of the per-layer error towards the output. 

Imagine that we would wish to compute these coefficients, in order to upper bound the compression error. 
One key issue is that, while the scaling coefficients $\alpha_l$ are compression-independent, the error coefficients $t_l = t_l(W_l^\star, \cQ_l)$ defined in Eqn.~(\ref{eq:t_l-def}) are specific to both the layer being compressed and \emph{to the quantizer used for compression}.  

However, we can \emph{remove this weight distribution dependence of the $t_l$ coefficients} by applying pre-processing to the weights. Specifically, it is well-known~\citep{doi:10.1137/060673096, suresh2017distributedmeanestimationlimited, chee2023quip} that multiplication of the layer weights with the Random Hadamard Transform (RHT) leads the weight distribution to closely match a Gaussian distribution, \emph{independently of the original weights}. 

Specifically, let us assume that we are applying the RHT to the weights, and then rounding to an arbitrary grid $\mathcal{G}_n^p$. The exact procedure is described in Algorithm~\ref{alg:vqrht}.
Since, post-RHT, the weight distribution is approximately Gaussian, we obtain that, in this case, the layer error coefficients $t_l$ will only depend on the chosen grid, and \emph{not on the original weights}. (Please see the proof of this fact in Appendix~\ref{app:mse_proof}.)
More specifically, $t_l^2$ approximately equals the per-dimension MSE of rounding the multivariate standard normal distribution to the grid $\mathcal{G}_n^p$, \textbf{which is constant given $n$ and $p$, and independent of the original weights}. 

In this context, Theorem~\ref{main:ppl-theorem} implies that if weights are Hadamard-transformed, then \textbf{MSE-optimized grids should be theoretically-optimal in terms of end-to-end model error}, given a fixed bit budget.

In the following subsection, we detail and expand these observations.

\begin{algorithm*}[t]
\caption{Vector Quantization with Random Hadamard Transform (\texttt{RHT-VQ})}\label{alg:vqrht}
{\small
\hspace*{\algorithmicindent} \textbf{Parameters:} grid $\mathcal{G}_n^p$ of $n$ elements of dimension $p$, scales group size $g$ that is a power of 2.

\hspace*{\algorithmicindent} \textbf{Input:} vector $\mathbf{w} = (\mathbf{w}_{\{1\}}, \dots,  \mathbf{w}_{\{D/g\}}) \in \mathbb{R}^D$, RHT seed $\xi$.

\hspace*{\algorithmicindent} \textbf{Output:} quantized vector $\mathbf{q^\dagger} \in \{1, 2, \dots, n\}^{D/p}$, scales vector $\mathbf{s} \in \mathbb{R}^{D/g}$.
    \begin{algorithmic}
    \State Sequentially partition $\mathbf{w} \in \R^D$ into $D/g$ subvectors $\mathbf{w}_{\{i\}} \in \R^g$, where $i = 1,\dots,D/g$     
    \For{$i = 1, \dots, D/g$}
        \State $s_i = \|\mathbf{w}_{\{i\}}\|_2$
        \State $\mathbf{w}^\dagger_{\{i\}} = \id{RandomHadamardTransform}(\mathbf{w}_{\{i\}} / s_i, \xi)$ \Comment{entries of $\mathbf{w}^\dagger_{\{i\}}$ are approx. from $\mathcal{N}(0, 1)$}
        \State $\mathbf{q}^\dagger_{\{i\}} =  \id{RoundToNearest}(\mathbf{w}_{\{i\}}^\dagger, \mathcal{G}_n^p)$ \Comment{Projecting $d$ sequential values together}
    \EndFor
    \State $\mathbf{q}^\dagger = \left[\mathbf{q}^\dagger_{\{1\}},...,\mathbf{q}^\dagger_{\{D/g\}}\right]$\Comment{$\dagger$ signifies that the vector is in Hadamard transformed space}
    \State $\mathbf{s} = \left[s_1,...,s_{D/g}\right] / \sqrt{g}$
    \end{algorithmic}
}
\end{algorithm*}

\subsection{MSE-optimal grids for LLM quantization}

Our previous insights lead to a simple alternative to Normal Float (NF) and Abnormal Float (AF) grids: after Hadamard rotations, we can quantize to a grid minimizing the $L2$ (MSE) quantization error. 
We call this approach \textbf{H}adamard \textbf{I}ncoherence and \textbf{G}aussian MSE-optimal \textbf{G}rid\textbf{S} (HIGGS). The algorithm combines the following components to achieve minimal quadratic quantization error: 1) Hadamard preprocessing of the quantized weights, 2) multi-dimensional (vector) quantization, and 3) Gaussian MSE-optimal quantization grids. 

Section~\ref{sec:linear_hadamard} described the exact quantity we need to optimize when choosing $\mathcal{G}_n^p$: the expected MSE of rounding the multivariate Normal distribution to $\mathcal{G}_n^p$. This problem has a rich history, and it can be solved optimally by the \citet{GaussianCase} algorithm arising in numerical PDEs. We use the same grid optimization procedure, as well as some pre-computed optimal grids. Applying those grids to Algorithm~\ref{alg:vqrht} constitutes Algorithm~\ref{alg:higgs}. It  important to note that the optimal grid only has to be computed once for any pair of  $n$ and $p$.

\begin{algorithm}
\caption{HIGGS Algorithm}\label{alg:higgs}
{\small
\hspace*{\algorithmicindent} \textbf{Parameters:} grid dimensions $n$ and $p$, scales group size $g$ that is a power of 2.

\hspace*{\algorithmicindent} \textbf{Input:} Algorithm~\ref{alg:vqrht} input.

\hspace*{\algorithmicindent} \textbf{Output:} Algorithm~\ref{alg:vqrht} output.
    \begin{algorithmic}
    \State $\mathcal{G}_n^p = \texttt{CLVQ}(n, p)$\Comment{\citep{GaussianCase}, computed once}
    \State $(\mathbf{q}^\dagger, \mathbf{s}) = \texttt{RHT-VQ}(\mathcal{G}_n^p, g, \mathbf{w}, \xi)$\Comment{Algorithm~\ref{alg:vqrht}}
    \end{algorithmic}
}
\end{algorithm}

To validate HIGGS, we compare it with other quantization grids, namely Normal Float (NF)~\citep{dettmers2023qloraefficientfinetuningquantized} and  Abnormal Float (AF)~\citep{yoshida2023nf4isntinformationtheoretically}. The results, sampled in
Figure~\ref{fig:compact_ppl_comparison},
indicate that HIGGS outperforms other grids in terms of output perplexity on WikiText-2~\citep{merity2016pointersentinelmixturemodels}. A more detailed comparison, including more baselines as well as zero-shot and few-shot tasks for both low and high bitwidth quantization can be found in Table~\ref{tab:method_comparison}.

\subsection{Practical Configurations}\label{sec:kernels}

HIGGS has a number of hyperparameters: the grid size $n$, the grid dimension $p$ and the group size $g$. Varying those, we can, in theory, achieve any per-parameter bitwidth. However, a number of practical considerations apply if we consider setups that can be efficiently implemented in practice:

\noindent\textbf{Constraint 1:} To optimize memory efficiency, $n$ must be a power of $2$. Since most modern architectures support data types with a minimum granularity of $1$ byte ($8$ bits), it is advantageous if $\log_2(n)$ is a multiple of $8$. When $\log_2(n)$ is less than $8$ but is still a power of $2$ (e.g., $n {=} 4, 16$), standard bit-packing methods can be used effectively. However, for cases where $\log_2(n)$ is not a power of $2$ (e.g., $n {=} 8, 64$), the data can be efficiently managed by partitioning it into multiple sections or bit-slices~\cite{xia2024fp6}.

\noindent\textbf{Constraint 2:} Grid memory access patterns can be sporadic. The ability to store the whole grid in low-latency memory would improve the performance of both decoding and matrix-multiplication operations. On modern GPUs, taking into account the usual shared-memory size of around 128Kb and $\frac{32}{k}\times$ replication to avoid bank conflicts, that would mean that the total number of points in the grid $2^{k \times p}$ can be at most $\approx 2^{10}$. Increasing the grid dimension at fixed bitwidth reduces the expected error, as seen in Figure~\ref{fig:compact_ppl_comparison}. This limitation creates a quantization error lower bound dictated by which dimensions we can use in practice.

The quantized matrix can be either restored via the Inverse Hadamard Transform or processed in the transformed space directly with virtually no matrix multiplication complexity overhead. (Refer to Appendix~\ref{app:processing_rotated} for theoretical and practical justification). Moreover, the Hadamard Transform functionality can be fully isolated from the matrix multiplication itself, allowing us to reuse existing lookup-table-based kernels for the latter.

\paragraph{FLUTE kernel.}
\label{paragraph:flute}
LLM decoding is typically memory-bound in the low-batch regime, making a GPU kernel that fuses dequantization and GEMM essential for achieving practical performance gains. A key component of such a fused HIGGS kernel is a primitive for vectorized indexing into the grid, implemented via a lookup table. The problem of implementing such kernel has been extensively studied by \citet{guo2024fastmatrixmultiplicationslookup}, resulting in their developing FLUTE: a scalar lookup table matrix multiplication kernel. FLUTE efficiently stores the lookup table in the GPU's shared memory, enabling faster on-chip memory access. Moreover, it efficiently optimizes the dot product computation patterns speeding up the processing of larger batch sizes.

The simplicity of the HIGGS design makes it compatible with FLUTE out of the box for grids where $p=1$. By adapting the kernels for vectorized lookups, we were able to unlock this functionality also for $p=2$. This extension allows us to handle vector-quantized data, supporting configurations $p \in \{1, 2\}$ and $b \in \{2, 3, 4\}$. In practice, $p=2$ is always preferable to $p=1$. We will refer to those setups as FLUTE grids.

Table~\ref{tab:kernel_speed} demonstrates the performance of the FLUTE lookup table approach for HIGGS ($p=2$), relative to MARLIN~\citep{frantar2024marlinmixedprecisionautoregressiveparallel} uniform quantization, Normal Float (NF)~\citep{dettmers2023qloraefficientfinetuningquantized}, AQLM~\citep{egiazarian2024extreme} and the QTIP~\citep{tseng2024qtip} specialized trellis quantized matrix multiplication kernels. For the QTIP and FLUTE kernels, these measurements already include the cost of the underlying Hadamard transforms. As we can see, FLUTE kernels achieve the best performance across a variety of bitwidths and batch sizes.

\paragraph{Constrained HIGGS.}
\label{paragraph:marlin}
To extend our method to higher bitwidths not supported by FLUTE (e.g. 8bit), we propose to reuse the existing uniform quantized matrix multiplication kernels \citep{ladder-osdi24, torchao}. To bridge the gap between HIGGS and those kernels in this high-density setting, we constrain the HIGGS grid to be uniform, essentially solving for positioning and scaling of uniform grids to minimize expected MSE over the Gaussian distribution. Such grids might be suboptimal in terms if MSE, but make up for it in terms of kernel support. In practice, we use this trick to allow for $p=1$, $b=8$ inference, to which we will refer as CH8.

\begin{table*}[t]
\centering
\caption{End-to-end throughput (tok/s, higher is better) comparison of quantized matrix multiplication kernels for Llama-3.1-8B at different bitwidths and batch sizes on an NVIDIA RTX 4090 GPU. We observe that MARLIN (which only supports uniform grids) and FLUTE are the only approaches that support speedups are batch sizes larger than 1, relative to FP16.}
\label{tab:kernel_speed}
\begin{tabular}{l|ccc|ccc|ccc}
\toprule
batch size & \multicolumn{3}{c|}{1}  & \multicolumn{3}{c|}{4}   & \multicolumn{3}{c}{16}  \\ \hline
FP16       & \multicolumn{3}{c|}{57} & \multicolumn{3}{c|}{224} & \multicolumn{3}{c}{862} \\ \midrule
\diagbox{kernel}{wbits}      & 2       & 3     & 4     & 2       & 3      & 4     & 2       & 3     & 4     \\ \hline
MARLIN     & -       & -     & 133   & -       & -      & 530   & -       & -     & 1873  \\
NF4        & -       & -     & 31    & -       & -      & 101   & -       & -     & 399   \\
AQLM       & 69     & -       & -    & 81      & -     & -       & 312    & -     & -     \\
QTIP       & 155     & 136   & 122   & 230     & 190    & 166   & 249     & 202   & 177   \\
FLUTE      & \textbf{173}     & \textbf{150}   & \textbf{139}   & \textbf{687}     & \textbf{592}    & \textbf{548}   & \textbf{2432}    & \textbf{2122}  & \textbf{1979}  \\
\bottomrule
\end{tabular}
\end{table*}

\subsection{Application to GPTQ}
\label{sec:gptq}

HIGGS can naturally be used in more sophisticated rounding schemes that utilize layer activations information to achieve smaller effect of quantization on model performance (1-shot quantization methods). Extension to GPTQ~\citep{frantar2022gptq}, one of the most popular such methods, in the form as we propose it can quickly be described as replacing the $\id{RoundToNearest}$ operation in Algorithm~\ref{alg:vqrht} with a different rounding operator that takes layer activations information into account. The resulting 1-shot quantized weights are structurally identical those obtained from Algorithm~\ref{alg:vqrht}.

In Table~\ref{tab:1-shot}, we present comparison of original GPTQ~\citep{frantar2022gptq}, AQLM~\citep{egiazarian2024extreme}, QuIP\#~\citep{tseng2024quipbetterllmquantization}, QTIP~\citep{tseng2024qtip} the and GPTQ extension of HIGGS (for details on this scheme, see  Appendix~\ref{app:exp_configurations}). Although the latter does not outperform the more complicated quantization schemes such as AQLM, QuIP\# and QTIP, we note that these methods use a more complex representation, and therefore provide very limited kernel support due to complexity of these representations. The GPTQ extension of HIGGS, on the other hand, can be mapped to  FLUTE kernels~\citep{guo2024fastmatrixmultiplicationslookup} achieving high throughput at a variety of setups. The full results can be examined in Table~\ref{tab:kernel_speed}, showing that we can reach close to 3x speedups in some configurations, and that this speedup is consistent across batch sizes from 1 to 16.

\begin{table*}
\centering
\caption{WikiText-2 PPL comparison of various 1-shot quantization methods for Llama-2-7b.}
\label{tab:1-shot}
\begin{tabular}{c|c|ccccc}
\hline
FP16                   &              & GPTQ  & AQLM  & QuIP\# & QTIP  & GPTQ+HIGGS ($p=2$) \\ \hline
\multirow{3}{*}{5.117} & wbits$\approx$2 & -     & 8.180 & 8.220  & 6.820 & 8.637      \\ \cline{2-7} 
                       & wbits$\approx$3 & 5.776 & -     & 5.600  & 5.380 & 5.559      \\ \cline{2-7} 
                       & wbits$\approx$4 & 5.254 & -     & 5.220  & 5.170 & 5.213      \\ \hline
\end{tabular}
\end{table*}

\section{Variable Bitwidth Quantization}\label{sec:variable_bitwidth}

The second application of the Linearity Theorem is in \emph{dynamic} quantization, i.e. choosing per-layer quantization bitwdiths that best reflect the “sensitivity” of different layers to quantization. Here, we leverage the observation that uniform biwidth compression might be far from optimal in terms of output error~\cite{owl}.  
Finding the optimal configuration is challenging due to exponentially-many possible solutions. 
We show that our quantization error model can efficiently find the optimal configuration for any target bitwidth, without having to evaluate all possible configurations.

\paragraph{Discrete Optimization Formulation.}
\label{sec:error-prediction}
Assume a natural setting in which we wish to quantize each layer $W_l^\star$ using one quantizer from a finite set of options $\{\cQ_1, \ldots, \cQ_J\},$
each with its own  error. Let $j_l \in \{1, \dots, J\}$ denote the selection of the quantizer for layer $l$. Assume that quantizer $\cQ_{j_l}$ corresponds to a specific bitwidth $b_{j_l}$ and specific induced  error $t_{l,j_l}^2$ from Eqn.~(\ref{eq:t_l-def}). 
We wish to find the optimal assignment \emph{minimizing perplexity error}, while matching a specific average bitwidth $b_{\max}$. 

\begin{algorithm}[t]
\caption{Error coefficient calibration}\label{alg:ecc}
{\small
\hspace*{\algorithmicindent} \textbf{Input:} Calibration constants $t_1,\dots, t_J$; pre-trained model $W^\star=(W_1^\star,\dots,W_L^\star)$ 

\hspace*{\algorithmicindent} \textbf{Output:} linear coefficients $\alpha_1,\dots,\alpha_L$

    \begin{algorithmic}

    \For{$l=1,\dots,L$}
        \For{$j=1,\dots,J$}
            \State $\Delta_{l,j} = \id{PPL}(W^\star(l,t_j))-\id{PPL}(W^\star)$
        \EndFor
        \State $\alpha_l = \arg\min \limits_{\alpha_l'} \sum_{j=1}^{J}  \left(\Delta_{l,j} - \alpha_l' \cdot t_j^2\right)^2$
    \EndFor
    \end{algorithmic}
}
\end{algorithm}

\paragraph{Problem Formulation.} Using error linearity from Theorem~\ref{main:ppl-theorem}, and the coefficients $\alpha_1,\dots,\alpha_L$ estimated via Algorithm~\ref{alg:ecc} as an input, this minimization problem can be written as 
\begin{equation}
\begin{gathered}
\label{eq:objective}
\min_{j_1,...,j_L} \sum_{l=1}^L \alpha_l \cdot t_{l,j_l}^2 \\
\sum_{l=1}^L b_{j_l} \cdot d^l \leq b_{\max} \cdot d\\
j_l \in \{1, \dots, J\} \text{ for all } l \in \{1,\dots,L\}\\ 
\end{gathered}
\end{equation}
where $b_{\max}$ is the target bitwidth. Recall that $d\eqdef \sum_{l=1}^L {d^l}$.

\paragraph{Estimating Scaling Coefficients.} 
Theorem~\ref{main:ppl-theorem} shows that the scaling coefficients $\alpha_l$ do {\em not} depend on the quantization method used. We use this fact to estimate these coefficients without using a real quantization method, and instead introduce a Gaussian noise insertion procedure, described in Appendix~\ref{app:gaussian-noise-insertion}. In this method we add normal noise to the weights that emulates the quantization error.
An advantage that we get is that we can accuracy regulate $t$ value.
Then we apply multiple ($J$) calibration noise levels  that are uniformly sampled from applicability region to each level to estimate the coefficients $\alpha_l$.
Algorithm~\ref{alg:ecc} describes the procedure. Note that $W^\star(l,t_j)$ represents the model $W^\star$ with all layers intact except for layer $l$, which is replaced by $\widehat{W}_l$ via the Gaussian noise insertion procedure with noise $t_j$, described in Eqn.~(\ref{eq:GNI}); see also (\ref{eq:889900}).  We found that sampling $J=15$ noise levels from linear theorem applicability range is enough to get accurate coefficients $\alpha_l$.

\paragraph{Measuring Grid Parameters.}
Grid bitwidths $b_j$ are inherent grid parameters (e.g. supported bit-widths) and known by design. The grid distortions $t_{l,j}^2$ can be measured explicitly by quantizing $W_l^\star$ with $\cQ_j$, creating a ``database'' of per-layer errors, across supported bitwidths. 

\paragraph{Solving the Problem.} We then solve the global minimization problem for the obtained $\alpha_{l}$, $b_j$ and $t_{lj}$ coefficients, acquiring the optimal quantization configuration for the given average bitwidth. Since the error function to be minimized \emph{is summable}, Equation~\ref{eq:objective} can be expressed as a linear programming problem, which can be solved using already existing optimization libraries.
Specifically, we use the CP-SAT solver from Google OR-Tools library~\citep{cpsatlp}, that can optimally solve this problem.
Figure \ref{fig:layerwise_ppl} demonstrates the practical dependence of the optimized objective on budged $b_{\max}$.

\paragraph{Data Free Dynamic Quantization.}
We also present a data-free dynamic quantization mode for our method.
Before, we used a calibration dataset to estimate the error coefficients $\alpha_l$, making it data-dependent.
To avoid the need for calibration dataset, we can change the metric we use in calibration described in Algorithm~\ref{alg:ecc} from perplexity on a calibration dataset to KL-divergence between pretrained and quantized models on randomly sampled text tokens.
We evaluate KL-divergence on 287k random tokens that are not shared between evaluations.

\begin{figure}[ht]
   \centering
   \includegraphics[width=1\linewidth]{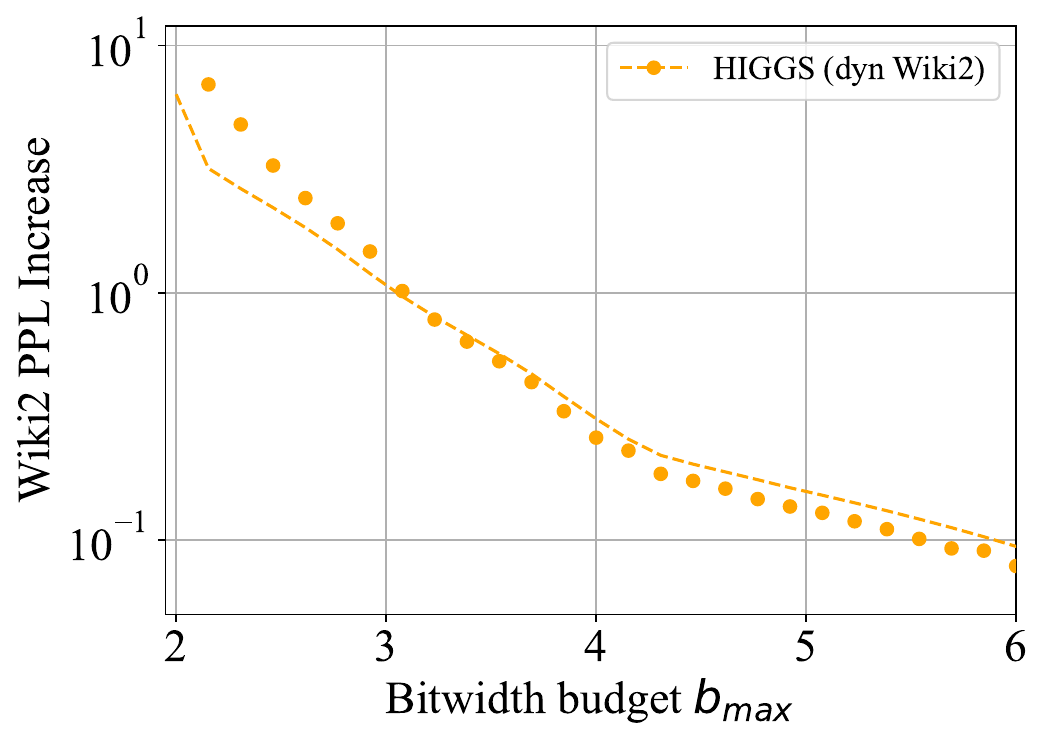}
   \caption{
   Demonstration of WikiText-2 PPL increase as a function of bitwidth budget $b_{\max}$ for layer-wise dynamic bitwidth quantization for Llama 3.1 8B. Dotted lines represent Linear Model predictions.}
   \label{fig:layerwise_ppl}
\end{figure}

\begin{table*}[ht!]
\centering
{\small
\begin{tabular}{l|l|l|llllll|l}
\toprule
Method & wbits & Wiki2 & ArcC & ArcE & PiQA & Wino & HellaS & Avg & MMLU \\
\midrule
FP16 & 16.00 & 5.607 & 51.28 & 81.52 & 80.03 & 73.72 & 60.01 & 69.31 & 65.35 \\ \hline 
AF & 3.25 & 8.056 & 43.94 & 75.25 & 77.53 & 69.38 & 52.91 & 63.80 & 53.15 \\
NF & 3.25 & 7.683 & 42.66 & 75.63 & 77.97 & 70.48 & 54.92 & 64.33 & 55.82 \\
HQQ & 3.25 & 7.317 & 43.17 & 76.14 & 78.24 & 68.98 & 55.37 & 64.38 & 56.39 \\
HIGGS (p=2) & 3.25 & 7.110 & 44.11 & 76.35 & 77.09 & 73.09 & 55.77 & 65.28 & 57.56 \\
HIGGS (p=3) & 3.25 & 6.807 & 44.71 & 77.95 & 77.75 & 71.11 & 57.01 & 65.71 & \bf{60.11} \\
HIGGS (p=4) & 3.25 & \bf{6.643} & 47.27 & 78.41 & 78.45 & 70.72 & 56.97 & \bf{66.36} & 59.88 \\ \hline 
\textcolor{gray}{GPTQ} & \textcolor{gray}{3.25} & \textcolor{gray}{7.133} & \textcolor{gray}{41.13} & \textcolor{gray}{72.81} & \textcolor{gray}{75.14} & \textcolor{gray}{71.51} & \textcolor{gray}{53.86} & \textcolor{gray}{62.89} & \textcolor{gray}{58.37} \\
HIGGS (dyn data-free) & 3.25 & \bf{6.388} & 47.10 & 79.12 & 78.78 & 71.59 & 57.09 & \bf{66.74} & \bf{61.62} \\ \hline 
AF & 4.02 & 6.194 & 46.84 & 78.54 & 79.16 & 73.95 & 58.28 & 67.35 & 61.47 \\
NF & 4.02 & 6.225 & 47.95 & 79.38 & 79.27 & 73.24 & 58.44 & 67.66 & 62.65 \\
HQQ & 4.02 & 8.057 & 46.84 & 78.16 & 77.91 & 70.17 & 55.44 & 65.70 & 57.72 \\
HIGGS (p=1) & 4.02 & 6.142 & 47.27 & 79.63 & 78.78 & 72.45 & 58.29 & 67.28 & 61.74 \\
HIGGS (p=2) & 4.02 & 6.015 & 48.29 & 81.06 & 79.54 & 73.95 & 58.54 & 68.28 & \bf{63.26} \\
HIGGS (p=3) & 4.02 & \bf{5.981} & 50.17 & 80.26 & 80.30 & 73.72 & 59.17 & \bf{68.73} & 62.83 \\ \hline 
\textcolor{gray}{GPTQ} & \textcolor{gray}{4.02} & \textcolor{gray}{6.238} & \textcolor{gray}{45.82} & \textcolor{gray}{78.66} & \textcolor{gray}{78.02} & \textcolor{gray}{72.53} & \textcolor{gray}{56.91} & \textcolor{gray}{66.39} & \textcolor{gray}{62.96} \\
HIGGS (dyn data-free) & 4.00 & \bf{5.910} & 49.23 & 80.98 & 79.38 & 72.85 & 59.00 & \bf{68.29} & \bf{63.86} \\ \hline 
AF & 4.25 & 5.952 & 49.57 & 80.85 & 79.27 & 74.27 & 59.13 & 68.62 & 63.20 \\
NF & 4.25 & 5.964 & 49.32 & 80.81 & 78.94 & 73.40 & 59.16 & 68.33 & 64.10 \\
HQQ & 4.25 & 5.944 & 50.09 & 81.44 & 79.76 & 73.88 & 59.44 & 68.92 & 63.70 \\
HIGGS (p=1) & 4.26 & 5.978 & 50.26 & 80.98 & 79.54 & 73.24 & 58.96 & 68.60 & 63.47 \\
HIGGS (p=2) & 4.26 & 5.908 & 50.60 & 81.48 & 79.38 & 74.19 & 59.17 & \bf{68.96} & 63.52 \\
HIGGS (p=3) & 4.25 & \bf{5.872} & 49.57 & 81.27 & 79.38 & 72.38 & 59.33 & 68.39 & \bf{64.24} \\ \hline 
\textcolor{gray}{GPTQ} & \textcolor{gray}{4.25} & \textcolor{gray}{5.923} & \textcolor{gray}{47.18} & \textcolor{gray}{79.59} & \textcolor{gray}{79.16} & \textcolor{gray}{72.22} & \textcolor{gray}{58.43} & \textcolor{gray}{67.32} & \textcolor{gray}{64.06} \\
HIGGS (dyn data-free) & 4.25 & \bf{5.831} & 50.43 & 81.27 & 79.43 & 72.85 & 59.33 & \bf{68.66} & \bf{64.06} \\
\bottomrule
\end{tabular}

}

\caption{Quantized Llama3.1 8B perplexity on WikiText-2~\citep{wikitext103}, accuracy on 5 zero-shot tasks~\cite{eval-harness}, average zero-shot accuracy, and 5-shot accuracy on MMLU. All quantization methods are data-free. Experiment configurations are described in Appendix~\ref{app:exp_configurations}.}
\label{tab:method_comparison}
\end{table*}

\section{Experiments}

\subsection{Error Model Validation}

One key point in the Theorem is that the result holds for \emph{sufficiently small} relative per-layer quantization errors $t_l$.
In this section, we seek to validate the fact that common bit-widths used in practice provide low enough compression error in order for the result to apply.
We conduct experiments for data-free weight quantization of the popular Llama 3.1 8B model with HIGGS quantization.
To evaluate the error model, we compare the predicted perplexity with the real perplexity of the quantized model.
For that, we uniformly quantize the model with different grid dimensions $p$ and grid sizes $n$.
We only use grids on pareto frontier of perplexity vs bitwidth with $1 \le p \le 5$ and $9 \le n \le 4096$.

Evaluated and predicted perplexities are shown in Figure~\ref{fig:perplexity-validation}.
We observe that the predicted perplexity is close to the real perplexity on relatively higher bitwidths ($b > 3.0$) and diverges on lower bitwidths, where the quantization error is higher.
Thus, we can use the error model in realistic bitwidth ranges.

\subsection{Methods Evaluation}

\paragraph{Methods.}
We compare HIGGS (Algorithm~\ref{alg:higgs}) with Normal Float (NF)~\citep{dettmers2023qloraefficientfinetuningquantized}, Abnormal Float (AF)~\citep{yoshida2023nf4isntinformationtheoretically} and HQQ~\citep{badri2023hqq}.
We provide detailed configurations and code sources for all the methods in Appendix~\ref{app:exp_configurations}. One important thing to note is that constant bitwidth HIGGS configurations were chosen to be as close to the default bitwidths of other method as possible with no regard for availability of kernels to run them. Dynamic bitwidth HIGGS results, however, limited the configurations to those mentioned in Section~\ref{sec:kernels}: FLUTE grids and CH8.

\paragraph{Models.}
We compare the aforementioned methods in application to quantization of models from the Llama 3~\citet{dubey2024llama3herdmodels} and Qwen2.5~\citep{bai2023qwen} families of models.
More specifically, we validate our findings on
Llama3.2 1B (Table~\ref{tab:llama3.2-1b}),
Llama3.2 3B (Table~\ref{tab:llama3.2-3b}),
Llama3.1 8B (Table~\ref{tab:method_comparison}),
Llama3.1 8B Instruct (Table~\ref{tab:llama3.1-8b-Instruct}),
Llama3.1 70B (Table~\ref{tab:llama3.1-70b}), and
Qwen2.5 7B (Table~\ref{tab:qwen2.5-7b}).

\paragraph{Metrics.} 
We report perplexity on WikiText-2~\cite{wikitext103} validation set. We measure zero-shot accuracy on WinoGrande~\cite{DBLP:journals/cacm/winogrande2021}, PiQA~\cite{tata2003piqa}, HellaSwag~\cite{DBLP:conf/acl/hellaswag2019}, ARC-easy and ARC-challenge~\cite{arc_allenai}, and report average zero-shot accuracy. We also report 5-shot accuracy on the MMLU~\citep{hendrycks2021measuringmassivemultitasklanguage} benchmark. Zero-shot and few-shot measurements are done via the LM Eval Harness~\cite{eval-harness}.

\begin{table}[h]
\centering
\caption{Comparison with 1-shot quantization methods for Llama 3.1 8B quantization.}
\label{tab:one_shot}
{\small
\begin{tabular}{l|c|c|c}
\toprule
Method                & wbits & Wiki2 & MMLU  \\ \midrule
FP16                  & 16    & 5.606 & 65.36 \\ \hline
GPTQ                  & 3.25  & 7.133 & 58.37 \\
HIGGS (dyn data-free) & 3.25  & 6.388 & 61.62 \\
HIGGS (dyn)           & 3.25  & 6.359 & 61.37 \\ \hline
GPTQ                  & 4.02  & 6.238 & 62.96 \\
AWQ                   & 4.02  & 6.228 & 62.88 \\
HIGGS (dyn data-free) & 4.00  & 5.910 & 63.86 \\
HIGGS (dyn)           & 4.00  & 5.870 & 63.69 \\ \hline
GPTQ                  & 4.25  & 5.923 & 64.05 \\
AWQ                   & 4.25  & 5.905 & 63.83 \\
HIGGS (dyn data-free) & 4.25  & 5.831 & 64.06 \\
HIGGS (dyn)           & 4.25  & 5.802 & 64.26 \\ \bottomrule
\end{tabular}
}
\end{table}

\paragraph{Uniform Bitwidths.}
We present the Llama 3.1 8B evaluations in Table~\ref{tab:method_comparison}. Evaluations for other models are present in Appendix~\ref{app:additional_evals}. For bitwidths around or below 4.0, We can see that HIGGS outperforms all existing $0$-shot compression methods even in the fixed-bitwidth applications.

\paragraph{Non-Uniform Bitwidths.}
We present the dynamic bitwidth results alongside the constant bitwidth results, in separate row groups in Tables~\ref{tab:method_comparison},~\ref{tab:llama3.2-1b}, ~\ref{tab:llama3.2-3b}, ~\ref{tab:llama3.1-8b-Instruct}, ~\ref{tab:llama3.1-70b}, and ~\ref{tab:qwen2.5-7b}.
More specifically, we present results for data-free dynamic bitwidth quantization expansion of our method (dyn data-free), described in detail in Section~\ref{sec:variable_bitwidth}.
For calibration we use $J=15$ noise values from linear theorem applicability range. We calibrate on 287k tokens from WikiText-2~\citep{wikitext103} train set.

\paragraph{Comparison With Data-Aware Methods.}
Additionally, we compare our data-free dynamic bitwidth HIGGS method with popular data-aware 1-shot quantization methods: GPTQ~\citep{frantar2022gptq} and AWQ~\citep{lin2023awq}. Alongside data-free method (dyn data-free), we present results for method calibrated on WikiText-2 PPL itself (dyn Wiki2). The results, shown in Table~\ref{tab:one_shot}, indicate that HIGGS consistently outperforms those quantization methods as well. Moreover, we observe little difference in few-shot performance between data-free and data-dependent method.

\section{Conclusions}

We have presented a new result relating the per-layer quantization error with the model's global error, and have applied this result to two problems in LLM quantization: accurate data-free quantization and  optimal non-uniform compression. Our approach leads to state-of-the-art performance for data-free quantization, and is compatible with efficient runtimes~\cite{bitsandbytes, guo2024fastmatrixmultiplicationslookup}. 
Remarkably, we observe that our approach is robust to being made completely data-free via random sampling; moreover, it appears to outperform popular calibration-based methods in the 3-4 bits/parameter range.



\section{Limitations}

One direction for improvement is validating the approach across several model architecture types (e.g. Mixture-of-Experts). However, we believe our result should be generalizable, as the quantization approach used is model-independent. One other limitation is the requirement to use Hadamard transforms for weight incoherence, which may add runtime overheads in some cases. However, it is known~\citep{chee2023quip} that these runtimes can be minimized, or that the corresponding matrices can even be eliminated by ``folding them into''  the previous layer~\citep{ashkboos2024quarotoutlierfree4bitinference}. We aim to investigate this in future work, as further enhancements to existing kernels such as FLUTE~\citep{guo2024fastmatrixmultiplicationslookup}.

\bibliography{custom}

\newpage
\onecolumn
\appendix
\section{Table of Notation}

\begin{table}[!h]
\centering
\footnotesize
\begin{tabular}{|c|l|c|}
\hline
\bf Notation & \bf Meaning & \bf Reference \\
\hline
$L$ & number of matrices the entries of which we want to quantize & \\
$W_l$ & matrix of floats corresponding to layer $l \in \{1,\dots,L\}$ & \\
$d_{in}^l \times d_{out}^l$ & dimensions of matrix $W_l$ & \\
$d_l$ & $=d_{in}^l \cdot d_{out}^l$ &  \\
$\cC_l$ & a (possibly randomized) compression mapping used to compress $W_l$ & \\  
$\cG_l$ & Gaussian noise insertion (a special type of compressor $\cC_l$) & (\ref{eq:GNI}) \\
$\widehat{W}_l$ &  $=\cC_l(W_l)$; compressed version of matrix $W_l$ & \\  
$W$   & $= (W_1,\dots,W_L)$ & \\
$W^\star$   & $= (W_1^\star,\dots,W_L^\star)$; weights of a pre-trained model & \\
$PPL(W)$ & the perplexity of the model associated with weights $W$ & \\
$d$ & $=\sum_{l=1}^L d_{in}^l \cdot d_{out}^l$; the total number of floats in the model we want to quantize & \\
$\mathcal{G}_n^p$ & a grid: a collection of $n$ vectors in $\R^p$ used for quantization by rounding to the nearest & \\
\hline
\end{tabular}
\caption{Selected frequently used notation.}
\end{table}

\section{Compression of Linear Layers}



\subsection{Compressing linear layers} Let $\cC_l: \R^{d^l_{in} \times d^l_{out}} \to \R^{d^l_{in} \times d^l_{out}}$ be a (possibly randomized) compression (e.g., sparsification and/or quantization) mechanism and let  $\widehat{W}_l \eqdef \cC_l(W_l) \in \R^{d^l_{in} \times d^l_{out}}$ represent a compressed/quantized version of $W_l$. Further, let
\begin{equation} \label{eq:t_l} t_l^2 = t_l^2(W_l,\cC_l) \eqdef \frac{ \Exp{\|\cC_l(W_l)-W_l\|_F^2 }}{\|W_l\|_F^2}, \quad l=1,\dots,L. \end{equation} 
That is, \begin{equation}\label{eq:t_l-second} \Exp{ \|\widehat{W}_l-W_l\|_F^2 } = t_l^2(W_l,\cC_l) \|W_l\|_F^2, \quad l=1,\dots,L.\end{equation}
 Here we remark that the highly studied class of {\em contractive} compressors is characterized by the inequality
\begin{equation} \label{eq:contractive} \Exp{ \|\cC_l(W_l)-W_l\|_F^2 } \leq (1-\alpha_l) \|W_l\|_F^2, \quad \forall W_l \in \R^{d^l_{in} \times d^l_{out}}, \quad \forall l \in \{1,\dots,L\},\end{equation}
which is assumed to hold for some $0<\alpha_l \leq 1$. If we apply such a contractive compressor to $W_l$, we get
\[t_l^2(W_l,\cC_l) \|W_l\|_F^2  \overset{\Cref{eq:t_l-second}}{=} \Exp{ \|\cC_l(W_l)-W_l\|_F^2 } \overset{\Cref{eq:contractive}}{\leq} (1-\alpha_l) \|W_l\|_F^2,\]
which means that $$t_l^2(W_l,\cC_l) \leq 1-\alpha_t.$$ In other words, unlike the contraction factor $1-\alpha_t$, which needs to hold universally for all matrices $W_l \in  \R^{d^l_{in} \times d^l_{out}}$, $t_l^2(W_l,\cC_l)$ is an ``instantaneous''  contraction factor corresponding to $W_l$ only, and as such, is always not worse, and typically much better.

\subsection{Gaussian noise insertion}\label{app:gaussian-noise-insertion}

We now describe a synthetic noise insertion procedure whose role is to mimic the error due to compression. Given a constant $t>0$, define \begin{equation}\label{eq:GNI} \cG_l(W_l,t) \eqdef  W_l + \frac{t \|W_l\|_F}{\sqrt{ d^l_{in} \cdot d^l_{out}}} \Sigma_l, \quad l=1,\dots,L, \end{equation} where the entries of $\Sigma_l \in \R^{d^l_{in} \times d^l_{out}}$ are all i.i.d.\ Gaussians with zero mean and unit variance. This means that the entries of $\cG_l(W_l,t)$ are also Gaussians, with mean determined by the entries of $W_l$, and variance $\sigma_l^2 \eqdef t^2 \|W_l\|_F^2$. This implies that
\begin{equation} \Exp{ \| \cG_l(W_l,t) - W_l\|_F^2 } = t^2 \|W_l\|_F^2,\quad l=1,\dots,L.\end{equation}

Comparing this to (\ref{eq:t_l-second}), we see that
$$t_l^2(W_l,\cG_l(\cdot,t)) = t^2.$$

\section{Empirical Approximation of Perplexity Around the Pre-trained Model} 

Let $PPL(W)$ denote the perplexity of the model corresponding to weights $W$, and let $W^\star$ be the weights corresponding to a pre-trained model. We have made the following (perhaps surprising!) observation through numerical experimentation: if we replace the pre-trained weights $W_l^\star$ of the $l^{\rm th}$ layer by $\widehat{W}_l = \cC_l(W^\star_l)$, where $\cC_l$ is a contractive compressor, then
\begin{equation} PPL(W_1^\star,\dots,W_{l-1}^\star, \widehat{W}_l, W_{l+1}^\star, \dots , W_L^\star) \approx PPL(W^\star) + \alpha_l t_l^2 \end{equation}
for all $t_l \in [0,\bar{t}_l]$, where $\bar{t}_l$ is ``small enough''\footnote{Add some comment on how small it is.}, and for some coefficient $\alpha_l>0$ computed via linear regression. If we specialize $$\widehat{W}_l=\cC_l(W_l^\star)\eqdef \cG_l(W_l^\star,t),$$ then the model $W^\star(l,t)$ described in Section~\ref{sec:variable_bitwidth} and Algorithm~\ref{alg:ecc} is defined as
\begin{equation} W^\star(l,t) \eqdef (W_1^\star,\dots,W_{l-1}^\star, \cG_l(W_l^\star, t), W_{l+1}^\star, \dots , W_L^\star).\label{eq:889900}\end{equation}

Moreover, a stronger experimental observation was made: if we replace the pre-trained weights of {\em all layers} by $\widehat{W}_l = \cC_l(W^\star_l)$, where $\cC_l$ is a contractive compressor, then
\begin{equation} PPL(\widehat{W}) \approx PPL(W^\star) + \sum_{l=1}^L \alpha_l t_l^2 \end{equation}
for $t_l \in [0,\bar{t}_l]$, where $\bar{t}_l$ is ``small enough'', and for some coefficient $\alpha_l>0$ computed via linear regression. 

This latter observation is formalized as Theorem~\ref{main:ppl-theorem}, and proved in the next section.

\section{Theoretical Approximation of Perplexity Around the Pre-trained Model}\label{app:ppl-theorem}

\subsection{Proof of Theorem~\ref{main:ppl-theorem}}
\begin{proof}
The proof follows from the material included in Section~ \ref{sec:D3}. The material in Sections~\ref{sec:D1} and \ref{sec:D2} provides a simplified treatment, and is included for clarity/pedagogical reasons.    
\end{proof}

\subsection{Approximating the mean of a smooth function perturbed around a local minimizer: univariate case} \label{sec:D1} 

Consider a sufficiently smooth function
 $\phi:\R\to \R$. Let $w^\star \in \R$ be a local minimizer of $\phi$, whence $\phi'(w^\star)=0$. Let  $\xi$ be any random variable with finite second moment: $$M_2 \eqdef \Exp{\xi^2} < +\infty.$$ From Taylor's approximation of $\phi$ around $w^\star$, we get
\begin{eqnarray*}\phi(w^\star + t |w^\star| \xi)  & \approx &  \phi(w^\star) +  \phi'(w^\star) t |w^\star| \xi + \frac{1}{2} \phi''(w^\star) t^2 |w^\star|^2 \xi^2  \\
&=& \phi(w^\star)  + \frac{1}{2} \phi''(w^\star) t^2 |w^\star|^2 \xi^2.
\end{eqnarray*}
Taking expectation on both sides, we get
\begin{eqnarray*} \Exp{\phi(w^\star + t |w^\star| \xi)} & \approx &  \Exp{\phi(w^\star)  + \frac{1}{2} \phi''(w^\star) t^2 |w^\star|^2 \xi^2} \\
&=&\phi(w^\star)  + \frac{1}{2} \phi''(w^\star) t^2 |w^\star|^2  \Exp{\xi^2} \\
&=&\phi(w^\star)  + \frac{1}{2} \phi''(w^\star) t^2 |w^\star|^2 \\
&= &\phi(w^\star) + \alpha t^2 M_2,
\end{eqnarray*}
where $\alpha \eqdef \frac{1}{2} \phi''(w^\star) |w^\star|^2 M_2 $.

Let us now make above approximation more precise. In order to do so, we will rely on two additional assumptions: 
\begin{itemize}
\item [(i)] the third derivative of $\phi$ is bounded by $B_3>0$: $|\phi'''(t) | \leq B_3$ for all $t\in \R$;
\item [(ii)] the third moment of $|\xi|$ is finite: $M_3 \eqdef \Exp{|\xi|^3}  < +\infty$.
\end{itemize}

Under these assumptions, there exists $\theta_\xi$ on the interval  defined by $w^\star$ and $w^\star + t |w^\star| \xi$ such that
\begin{eqnarray*}\phi(w^\star + t |w^\star| \xi) &=&  \phi(w^\star) +  \phi'(w^\star) t |w^\star| \xi + \frac{1}{2} \phi''(w^\star) t^2 |w^\star|^2 \xi^2  + \frac{1}{6}\phi'''(\theta_\xi) t^3 |w^\star|^3 \xi^3 \\
&=& \phi(w^\star) + \frac{1}{2} \phi''(w^\star) t^2 |w^\star|^2 \xi^2  + \frac{1}{6}\phi'''(\theta_\xi) t^3 |w^\star|^3 \xi^3. 
\end{eqnarray*}

By taking expectation on both sides, we get
$$\Exp{\phi(w^\star + t |w^\star| \xi)} = \phi(w^\star) + \frac{1}{2} \phi''(w^\star) t^2 |w^\star|^2 M_2  + \Exp{\frac{1}{6}\phi'''(\theta_\xi) t^3 |w^\star|^3 \xi^3},$$
from which we get the estimate
\begin{eqnarray*}
\left| \Exp{\phi(w^\star + t |w^\star| \xi)} - \left( \phi(w^\star) + \frac{1}{2} \phi''(w^\star) t^2 |w^\star|^2 M_2 \right) \right| &\leq & \left|\Exp{\frac{1}{6}\phi'''(\theta_\xi) t^3 |w^\star|^3 \xi^3} \right| \\
&\leq &\Exp{\left|\frac{1}{6}\phi'''(\theta_\xi) t^3 |w^\star|^3 \xi^3\right|}  \\
&=&\frac{1}{6} |t|^3  |w^\star|^3 \Exp{\left|\phi'''(\theta_\xi)\right|   |\xi|^3}  \\
&\leq& \frac{1}{6} |t|^3  |w^\star|^3 B_3 M_3. 
\end{eqnarray*}

\subsection{Approximating the mean of a smooth function perturbed around a local minimizer: multivariate case}\label{sec:D2} 

Let's now consider the multivariate case.  Consider a sufficiently smooth function
 $\phi:\R^d\to \R$. Let ${\bf w}^\star \in \R^d$ be a local minimizer of $\phi$, whence $\nabla \phi({\bf w}^\star)=0$. Let  $\xi =(\xi_1,\dots,\xi_d)$ be a random vector such that the second moment $$M_{2,i} \eqdef \Exp{\xi_i^2}$$ is finite for all $i\in \{1,\dots,d\}$. Let $T = Diag(t_1 , \dots, t_d )$ and $D = Diag(|{\bf w}_1^\star|, \dots,  |{\bf w}_d^\star|)$, where $t_1,\dots,t_d>0$.  From Taylor's approximation of $\phi$ around ${\bf w}^\star$, we get
\begin{eqnarray*}\phi({\bf w}^\star +  D T \xi)  & \approx &  \phi({\bf w}^\star) +  \left\langle \nabla \phi({\bf w}^\star), DT \xi \right \rangle + \frac{1}{2} \left\langle \nabla^2 \phi({\bf w}^\star) DT \xi , DT \xi \right \rangle  \notag \\
&=& \phi({\bf w}^\star)  + \frac{1}{2} \left\langle \nabla^2 \phi({\bf w}^\star) DT \xi , DT \xi \right \rangle \notag  \\
&= & \phi({\bf w}^\star)  + \frac{1}{2} \left\langle D \nabla^2 \phi({\bf w}^\star) DT \xi , T \xi \right \rangle . \label{eq:Taylor_approx_1}
\end{eqnarray*}

Let's assume that $D \nabla^2 \phi({\bf w}^\star) D $ is approximately diagonal, i.e., there exists a diagonal matrix $Z=Diag(z_1,\dots,z_d)$ with $z_l>0$ for all $l$ such that $D \nabla^2 \phi({\bf w}^\star) D \approx Z$. Under this assumption, \begin{equation}\label{eq:expansion-2}\left\langle D \nabla^2 \phi({\bf w}^\star) DT \xi , T \xi \right \rangle \approx \left\langle Z T \xi , T \xi \right \rangle =  \sum_{i=1}^d z_i t_i^2 \xi_i^2,\end{equation} 
and hence
\begin{eqnarray*}\Exp{\phi({\bf w}^\star +  D T \xi)}  & \approx & 
\Exp{ \phi({\bf w}^\star)  + \frac{1}{2} \left\langle ZT \xi , T \xi \right \rangle } \\
&=& \phi({\bf w}^\star)  + \frac{1}{2} \Exp{\left\langle ZT \xi , T \xi \right \rangle } \\
&\overset{\Cref{eq:expansion-2}}{=}& \phi({\bf w}^\star)  + \frac{1}{2} \Exp{ \sum_{i=1}^d z_i t_i^2 \xi_i^2} \\
&=&\phi({\bf w}^\star)  + \frac{1}{2}  \sum_{i=1}^d z_i t_i^2 \Exp{\xi_i^2} \\
&= & \phi({\bf w}^\star)  + \frac{1}{2}  \sum_{i=1}^d z_i t_i^2 M_{2,i}\\
&=& \phi({\bf w}^\star)  +  \sum_{i=1}^d \alpha_i t_i^2,
\end{eqnarray*}
where $\alpha_i \eqdef \frac{z_i M_{2,i}}{2}$. 


\subsection{Approximating the mean of a smooth function perturbed around a local minimizer: multivariate block case}\label{sec:D3} Let's now extend the last result to the multivariate {\em block} setting.  Consider a sufficiently smooth function
 $\phi:\R^d\to \R$, where $d=d_1+\dots+d_L$. For ${\bf w}\in \R^d$ we shall write ${\bf w} = ({\bf w}_1,\dots,{\bf w}_L)$, where ${\bf w}_l \in \R^{d_l}$ is the $l^{\rm th}$ block of vector ${\bf w}$. Let ${\bf w}^\star = ({\bf w}^{\star,1}, \dots {\bf w}^{\star,L}) \in \R^d$ be a local minimizer of $\phi$, whence \begin{equation}\label{eq:grad=0-3}\nabla \phi({\bf w}^\star)=0.\end{equation} Let  $\xi  = (\xi_1,\dots,\xi_L)$ be a random vector in $\R^d$ such that the block $\xi_l \in \R^{d_l}$ has finite second moment: \begin{equation} \label{eq:M_{2,l}} M_{2,l} \eqdef \Exp{\|\xi_l\|_2^2}\end{equation} for all $l\in \{1,\dots,L\}$. 
 
For each $l\in \{1,\dots,L\}$, let $T_l \eqdef t_l I_{d_l} \in \R^{d_l\times d_l}$ and $D_l \eqdef \|{\bf w}^{\star,l}\|_2 I_{d_l} \in \R^{d_l\times d_l}$. Further, let $T\eqdef Diag(T_1,\dots, T_L)\in \R^{d \times d}$ and $D \eqdef Diag(D_1,\dots,D_L)\in \R^{d \times d}$. That is, $D$ (resp.\ $T$) be the blog-diagonal matrix whose blocks are formed from the matrices $\{D_l\}$ (resp.\ $\{T_l\}$).   From Taylor's approximation of $\phi$ around ${\bf w}^\star$, we get
\begin{eqnarray*}\phi({\bf w}^\star +  D T \xi)  & \approx &  \phi({\bf w}^\star) +  \left\langle \nabla \phi({\bf w}^\star), DT \xi \right \rangle + \frac{1}{2} \left\langle \nabla^2 \phi({\bf w}^\star) DT \xi , DT \xi \right \rangle  \\
&\overset{\Cref{eq:grad=0-3}}{=}& \phi({\bf w}^\star)  + \frac{1}{2} \left\langle \nabla^2 \phi({\bf w}^\star) DT \xi , DT \xi \right \rangle \\
&= & \phi({\bf w}^\star)  + \frac{1}{2} \left\langle D \nabla^2 \phi({\bf w}^\star) DT \xi , T \xi \right \rangle .
\end{eqnarray*}

Let's assume that $D \nabla^2 \phi({\bf w}^\star) D $ is approximately block-diagonal, i.e., there exists a block-diagonal matrix $Z=Diag(Z_1,\dots,Z_L)$ such that $D \nabla^2 \phi({\bf w}^\star) D \approx Z$. Moreover, assume that $Z_l = z_l I_{d_l}$ for some $z_l >0$ and all $l \in \{1,\dots,L\}$. Under this assumption, \begin{equation} \label{eq:diag-3}\left\langle D \nabla^2 \phi({\bf w}^\star) DT \xi , T \xi \right \rangle \approx \left\langle Z T \xi , T \xi \right \rangle = \sum_{l=1}^L \left\langle Z_l T_l \xi_l , T_l \xi_l  \right \rangle = \sum_{l=1}^L t_l^2 z_l \|\xi_l \|_2^2,\end{equation}
and hence
\begin{eqnarray*}\Exp{\phi({\bf w}^\star +  D T \xi)}  & \approx & 
\Exp{ \phi({\bf w}^\star)  + \frac{1}{2} \left\langle ZT \xi , T \xi \right \rangle } \\
&=& \phi({\bf w}^\star)  + \frac{1}{2} \Exp{\left\langle ZT \xi , T \xi \right \rangle } \\
&\overset{\Cref{eq:diag-3}}{=}& \phi({\bf w}^\star)  + \frac{1}{2} \Exp{ \sum_{l=1}^L t_l^2 z_l \|\xi_l \|_2^2} \\
&=&\phi({\bf w}^\star)  + \frac{1}{2}  \sum_{l=1}^L t_l^2 z_l \Exp{\|\xi_l \|_2^2} \\
&\overset{\Cref{eq:M_{2,l}}}{=} & \phi({\bf w}^\star)  + \frac{1}{2}  \sum_{l=1}^L t_l^2 z_l M_{2,l} \\
&=& \phi({\bf w}^\star)  +  \sum_{l=1}^L \alpha_l t_l^2,
\end{eqnarray*}
where $\alpha_l \eqdef \frac{z_l M_{2,l}}{2}$. 


\newtheorem{corollary}{Corollary}[theorem]

\section{Experimental Justification of Assumption~\ref{ass:reg}}
\label{sec:assmp_3_justification}

In this section we justify Assumption~\ref{ass:reg} by analyzing the Hessian of the OPT-125M model \citep{zhang2022opt}.
This model consists of $12$ blocks, each containing multiple matrices. Computing the full Hessian for even a single matrix from the self attention of the first layer is infeasible due to its size -- $768\times768 = 589,824$ parameters, leading to a Hessian with around $400$ Billion entries.

Given these constraints, we focused on a smaller scope of $t\in \N$ parameters of the module from every layer.
Figure~\ref{fig:dHd_short} illustrates the structure of the product in Equation~(\ref{eq:product_ass3}).

\begin{figure}[ht]
    \centering
    \vspace{-7px}
    \includegraphics[width=0.85\linewidth]{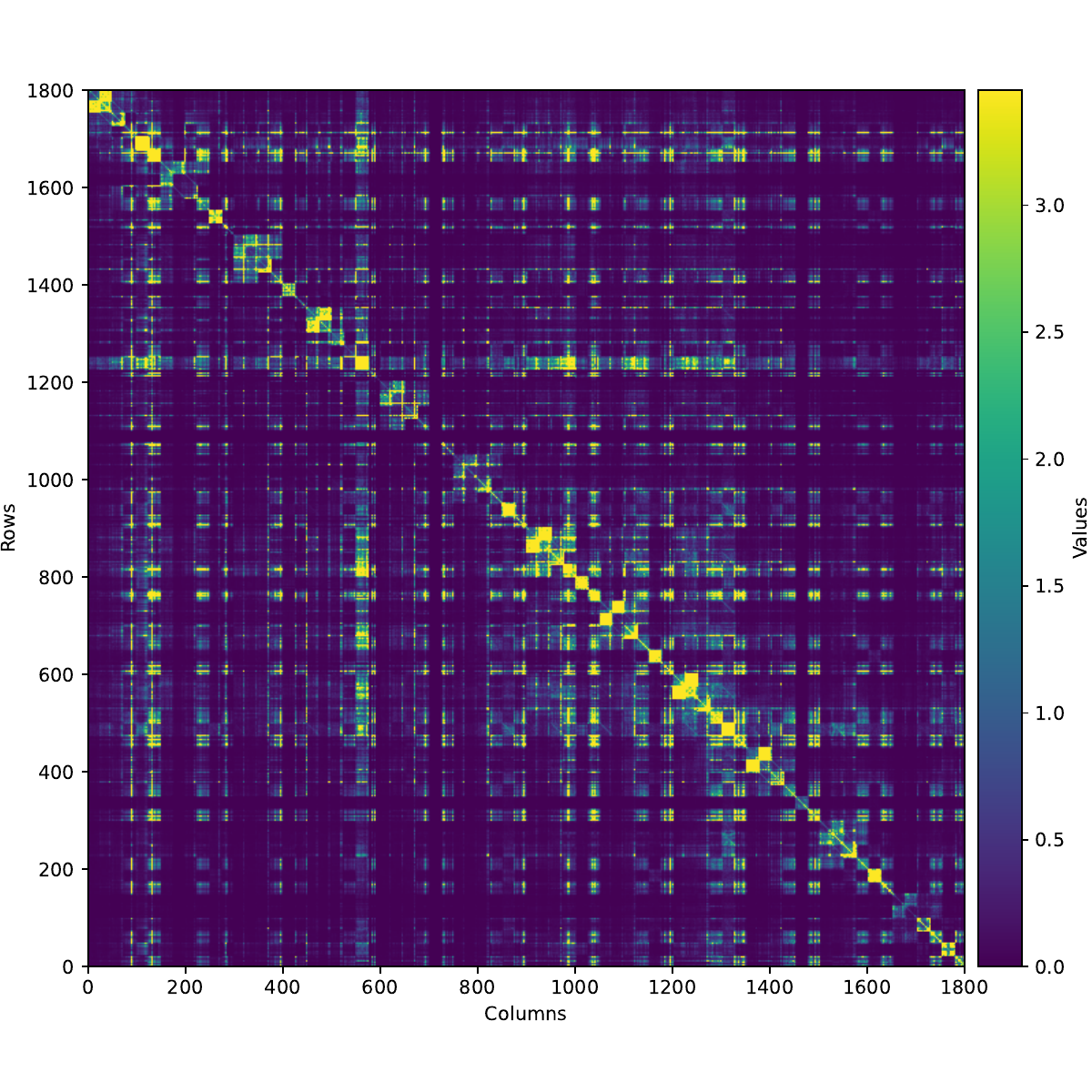}
    \captionsetup{width=1.0\linewidth}
    \captionsetup{aboveskip=0pt, belowskip=12pt}
    \caption{Visual representation of part of the product $|D^\star \nabla^2 \phi\left( {\cal R}(W^\star)\right) D^\star|$. The diagonal-dominant structure justifies the Assumption~\ref{ass:reg}. Please see Appendix Figure~\ref{fig:dHd}) for additional subsets.}
    \label{fig:dHd_short}
\end{figure}

We can see a clear diagonal structure of the product $D^\star \nabla^2 \phi\left( {\cal R}(W^\star)\right) D^\star$.
We observed the same diagonal structure under several other parameter samples for OPT-125M for the product $D^\star \nabla^2 \phi\left( {\cal R}(W^\star)\right) D^\star$ in Assumption~\ref{ass:reg}. Hence,  this suggests that Assumption~\ref{ass:reg} is satisfied.

\subsection{Experimental setup}

To justify the Assumption~\ref{ass:reg}  we have made various experiments with a OPT-125M model \citep{zhang2022opt}. The OPT-125M model consists of $B=12$ subsequent blocks, each containing multiple layers (matrices). For every block $i\in\{1, \cdots, B\}$ we have the following linear layers:
\begin{itemize}
    \item Self attention:
    \begin{itemize}
        \item q\_proj: $Q_i \in \R^{768 \times 768}$
        \item v\_proj: $V_i \in \R^{768 \times 768}$
        \item k\_proj: $K_i \in \R^{768 \times 768}$
        \item out\_proj: $O_i \in \R^{768 \times 768}$
    \end{itemize}
    \item First fully connected layer: $C_i \in \R^{768 \times 3072}$
    \item Second fully connected layer: $S_i \in \R^{3072 \times 768}$
\end{itemize}

Let us define $m\in \N$ -- the number of layers in a blocks. For OPT-125M, $m=6$. Since we have $B=12$ blocks with $m=6$ layers in each, in total there are $L=Bm=72$ layers, so $l\in\{1,\cdots,72\}$.

\subsection{Single layer from a single block}

Computing the full Hessian $\nabla^2 \phi\left( {\cal R}(Q_1)\right)$ for even a single $Q_1 \in \R^{768\times768}$ matrix from a self attention of the first block is infeasible due to its size -- $768\times768$ results in $589,824$ parameters, leading to a Hessian with around $400$ billion elements.

Given these constraints, we decided to focus on a smaller scope of $t\in \N$ parameters of the module. Let us consider some layer $W_l \in \R^{d^l_{in} \times d^l_{out}}$.

In this section we will use the following notation for entries of vectors and matrices: for any vector $r$, we denote $[r]_i$ as the $i$-th element of the vector $r$. Similarly, for any matrix $A$, we denote $[A]_{ij}$ as the element of $A$ in the intersection of the $i$-th row and $j$-th column.

Now let us consider the first $t \in \{1, \cdots, d^l_{in} \cdot d^l_{out}\}$ entries of the matrix $W_l$ -- the vector $\widetilde{W}_l = [\mathcal{R}(W_l)]_{:t}$, $\widetilde{W}_l \in \R^{t}$. With these notation we define a concatenation function $\psi_{W_l}: \R^{t} \to \R^{d^l_{in} \cdot d^l_{out}}$:
\begin{equation*}
    \psi_{W_l}(\widetilde{W}_l) = \left(\widetilde{W}_l, [\mathcal{R}(W_l)]_{t:}\right).
\end{equation*}
Finally, let us define the perplexity function $\widetilde{\phi}_{W_l}: \R^{t} \to \R$ as a function of $\widetilde{W}_l$:
\begin{equation*}
    \widetilde{\phi}_{W_l}(\widetilde{W}_l) \eqdef \phi\left(\psi_{W_l}(\widetilde{W}_l)\right).
\end{equation*}

For a subset of parameters we define a matrix $\widetilde{D}_l \in \R^{t\times t}$: $\widetilde{D}_l \eqdef \|W_l\|_F I_{t} \in \R^{t\times t}$. Then we define the analog of the product (\ref{eq:product_ass3}) for a subset of $t$ parameters:

\begin{equation*}
    \widetilde{D}_l \nabla^2 \widetilde{\phi}_{W_l}(\widetilde{W}_l) \widetilde{D}_l \approx \widetilde{Z}_l,
\end{equation*}
where $\widetilde{Z}_l = z_l I_t, z_l > 0$.

\subsection{All layers from a single block}
If we aim to consider all sub-modules of a single block $i\in\{1, \cdots, B\}$, then we define $\widetilde{W}^i \eqdef (\widetilde{W}_{mi+1}, \cdots, \widetilde{W}_{mi+m})$, where $\widetilde{W}_{mi+s} \in \R^t$ is the subset of the first $t$ parameters of the $s$-th sub-module from the $i$-th block: $\widetilde{W}_{mi+s} = [\mathcal{R}(W_{mi+s})]_{:t}$, $\widetilde{W}_{mi+s} \in \R^{t}$. In this case, the perplexity function $\widetilde{\phi}^i: \R^{tm} \to \R$ can be represented the the following way:
\begin{equation*}
    \widetilde{\phi}^i(\widetilde{W}^i) \eqdef \phi\left(\psi_{W_{mi+1}}(\widetilde{W}_{mi+1}), \cdots, \psi_{W_{mi+m}}(\widetilde{W}_{mi+m})\right).
\end{equation*}

Then we define
$$\widetilde{D}^i \eqdef Diag(\widetilde{D}_{mi+1},\dots,\widetilde{D}_{mi+m})\in \R^{tm \times tm},$$
where $\widetilde{D}_{mi+s} \eqdef \|W_{mi+s}\|_F I_{t} \in \R^{t \times t}$. Then the analog of the product (\ref{eq:product_ass3}) will be 
$$
\widetilde{D}^i \nabla^2 \widetilde{\phi}^i(\widetilde{W}^i) \widetilde{D}^i \approx \widetilde{Z}^i,
$$ where $\widetilde{Z}^i=Diag(\widetilde{Z}_{mi+1},\dots,\widetilde{Z}_{mi+m}) \in \R^{tm \times tm}$.

\subsection{Single sub-module from all blocks}

If we aim to consider a single sub-module $Q_i$ for all $i\in\{1,\cdots, B\}$, then we define $\widetilde{Q} \eqdef (\widetilde{W}_1, \widetilde{W}_{m+1}, \widetilde{W}_{2m+1}, \cdots, \widetilde{W}_{Bm+1})$, where $\widetilde{W}_{im+1} \in \R^t$ is the subset of the first $t$ parameters of the first sub-module from the $i$-th block: $\widetilde{W}_{im+1} = [\mathcal{R}(Q_i)]_{:t}$, $\widetilde{W}_{im+1} \in \R^{t}$. In this case, the perplexity function $\widetilde{\phi}_Q: \R^{Bt} \to \R$ can be represented the the following way:
\begin{equation*}
    \widetilde{\phi}_Q(\widetilde{Q}) \eqdef \phi\left(\psi_{W_1}(\widetilde{W}_1),  \psi_{W_{m+1}}(\widetilde{W}_{m+1}), \psi_{W_{2m+1}}(\widetilde{W}_{2m+1}), \cdots, \psi_{W_{Bm+1}}(\widetilde{W}_{Bm+1})\right).
\end{equation*}

Then we define
$$\widetilde{D}_Q \eqdef Diag(\widetilde{D}_{1}, \widetilde{D}_{m+1}, \widetilde{D}_{2m+1}, \dots,\widetilde{D}_{Bm + 1})\in \R^{Bt \times Bt},$$
where $\widetilde{D}_{im+1} \eqdef \|Q_{i}\|_F I_{t} \in \R^{t\times t}$, $Q_i$ is the first module from the $i$-th layer. Then the analog of the product (\ref{eq:product_ass3}) will be 
$$
\widetilde{D}_Q \nabla^2 \widetilde{\phi}_Q(\widetilde{Q}) \widetilde{D}_Q \approx \widetilde{Z}_Q,
$$ where $\widetilde{Z}_Q=Diag(\widetilde{Z}_{1}, \widetilde{Z}_{m+1}, \widetilde{Z}_{2m+1}, \dots, \widetilde{Z}_{Bm+1}) \in \R^{Bt \times Bt}$.

\subsection{All layers from all blocks}

Finally, if we aim to consider all layers and all sub-modules, then we define $\widetilde{W} \eqdef (\widetilde{W}_1, \cdots, \widetilde{W}_L)$. The perplexity function $\widetilde{\phi} : \R^{tL} \to \R$ in this case will be
\begin{equation*}
    \widetilde{\phi}(\widetilde{W}) \eqdef \phi\left(\psi_{W_1}(\widetilde{W}_1), \cdots, \psi_{W_L}(\widetilde{W}_L)\right)
\end{equation*}
and $\widetilde{D}^{\star} \eqdef Diag(\widetilde{D}_1,\dots,\widetilde{D}_L)\in \R^{tL \times tL}$, the analog of the product (\ref{eq:product_ass3}) will be
$$
\widetilde{D}^\star \nabla^2 \widetilde{\phi}(\widetilde{W}) \widetilde{D}^\star \approx \widetilde{Z},
$$ where $\widetilde{Z}=Diag(\widetilde{Z}_1,\dots,\widetilde{Z}_L) \in \R^{tL \times tL}$.

\subsection{Experimental results}

Firstly, let us consider $t=768$ parameters from only the first layer's $Q_1 \in \R^{768\times768}$. On (Fig.~\ref{fig:dHd_full_line_q_proj}) we can see a diagonal structure of the product $ \widetilde{D}_1 \nabla^2 \widetilde{\phi}_{W_1}(\widetilde{W}_1) \widetilde{D}_1  \in \R^{768 \times 768}$. In this particular case $\widetilde{W}_1$ is the first row of the matrix $Q_1$.

On the next step, we considered all sub-modules of the first layer -- from each matrix $W \in \{Q_1, V_1, K_1, O_1, C_{1}, S_{1}\}$, we selected $t=300$ parameters (Fig.~\ref{fig:dHd_first_layer_all_matrices}). We can see a clear diagonal structure of the product $\widetilde{D}^1 \nabla^2 \widetilde{\phi}^1(\widetilde{W}^1) \widetilde{D}^1 \in \R^{1800 \times 1800}$.

Then we expanded the Hessian computation to include parameters from multiple layers -- from each layer, we selected $t=150$ parameters of the matrix $Q_i$ and repeated this for $i\in \{1,\cdots, B\}$, yielding a product $\widetilde{D}_Q \nabla^2 \widetilde{\phi}_Q(\widetilde{Q}) \widetilde{D}_Q \in \R^{1800 \times 1800}$ (Fig.~\ref{fig:dHd_all_layers_first_matrix}).

Finally, we expanded the Hessian computation to include parameters from all layers and all sub-modules -- from each layer $W_l$, we selected $t=25$ parameters (Fig.~\ref{fig:dHd_all_layers_all_matrices}). As before, we can see a clear diagonal structure of the product $\widetilde{D}^\star \nabla^2 \widetilde{\phi}(\widetilde{W}) \widetilde{D}^\star \in \R^{1800 \times 1800}$.

In all considered subsets of parameters of OPT-125M we observed a diagonal structure of the product $s$ from Assumption~\ref{ass:reg}. Hence, we have fairly strong reasons to believe that Assumption~\ref{ass:reg} is satisfied for LLM's.

\subsection{Implementation details}

To be able to compute the Hessian for only a subset of the network's parameters, we manually defined a perplexity function as a function of the specific subset of parameters. We used this with the PyTorch \citep{paszke2019pytorch} autograd routines to compute the Hessian quickly and accurately.

Note that Hessian computation induces significant memory consumption when using larger batch sizes. A larger batch size is crucial for accurate perplexity computation, and therefore for accurate Hessian computation (for instance, a batch size of $140$ yields a WikiText-2 perplexity for OPT-125M of $27.65$, while a batch size of $4$ results in a WikiText-2 perplexity of $30.06$). To mitigate the memory overflow problem, we modified the perplexity function to exhibit an additive property (Sec.~\ref{sec:hessian_for_large_bs}). This means that we can compute the Hessian for the full batch by averaging Hessians, computed over smaller batches. With this adjustment, we were able to use PyTorch's autograd routine to compute the Hessian on a full batch size without encountering memory overflow issues.

\subsection{How to compute a Hessian for large batch sizes}
\label{sec:hessian_for_large_bs}

The perplexity function is computed by the following sequence of events:

\begin{enumerate}
    \item Take the input $X \in \R^{b \times l}$ and compute the output of the model by the function $f:\R^{b\times l} \to \R^{b\times l\times n}$, where $b$ is a batch size and $l$ is an output sequence length and $n$ is the size of the embedding space. For OPT-125M, $l=2048, n=50272$. Elements of $f(X)$ are called logits and represent the probability for each word to be the next token in the output sequence.
    \item Compute the Cross Entropy Loss of the output by the function $g:\R^{b\times l\times n} \to \R^{b}$, $$g(f(X)) \eqdef \text{nn.CrossEntropyLoss}(f(X)).$$
    \item Average the Cross Entropy Loss for all elements from a batch: $$\Bar{c} \eqdef \frac{1}{b}\sum_{i=1}^{b}{c_i}.$$
    \item Compute the Perplexity: $$PPL(\Bar{c}) \eqdef e^{\Bar{c}}.$$
\end{enumerate}

Let us change steps (3) and (4): on the step (3) we will not divide the sum by $b$, so we define $$\Bar{c}' = \sum_{i=1}^{b}{c_i},$$ on the step (4) we will not use the exponential function -- instead we will use the identity function: 
\begin{equation}
    \label{eq:new_def_of_ppl}
    PPL'(\Bar{c}') \eqdef \Bar{c}'.
\end{equation}

\begin{figure}[ht]
    \centering
    \begin{subfigure}{0.49\linewidth}
        \includegraphics[width=0.99\linewidth]{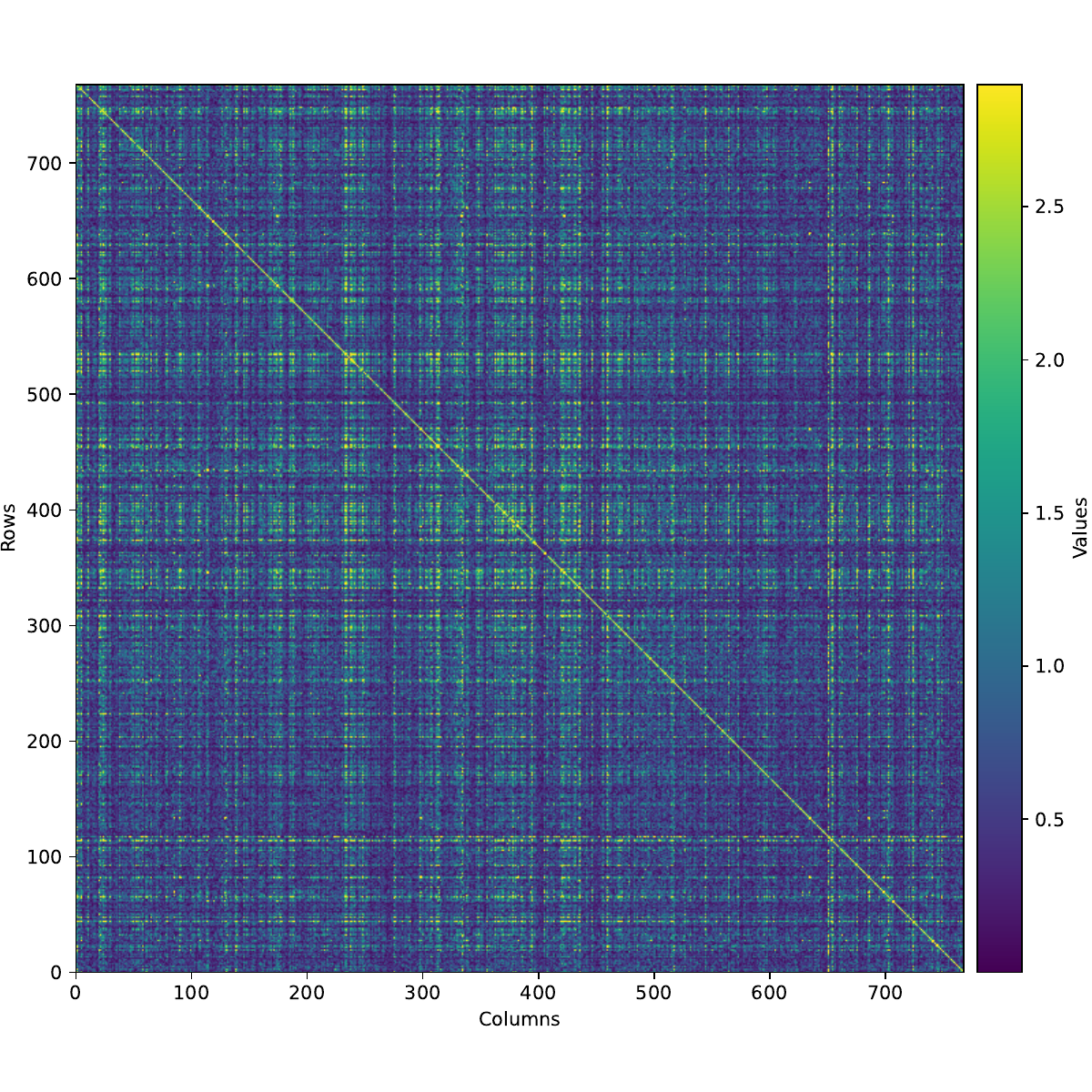}
        \captionsetup{width=.9\linewidth}
        \captionsetup{aboveskip=0pt, belowskip=12pt}
        \caption{$|\widetilde{D}_1 \nabla^2 \widetilde{\phi}_{W_1}(\widetilde{W}_1) \widetilde{D}_1|$, $t=768$ parameters.}
        \label{fig:dHd_full_line_q_proj}
    \end{subfigure}
    \begin{subfigure}{0.49\linewidth}
        \includegraphics[width=0.99\linewidth]{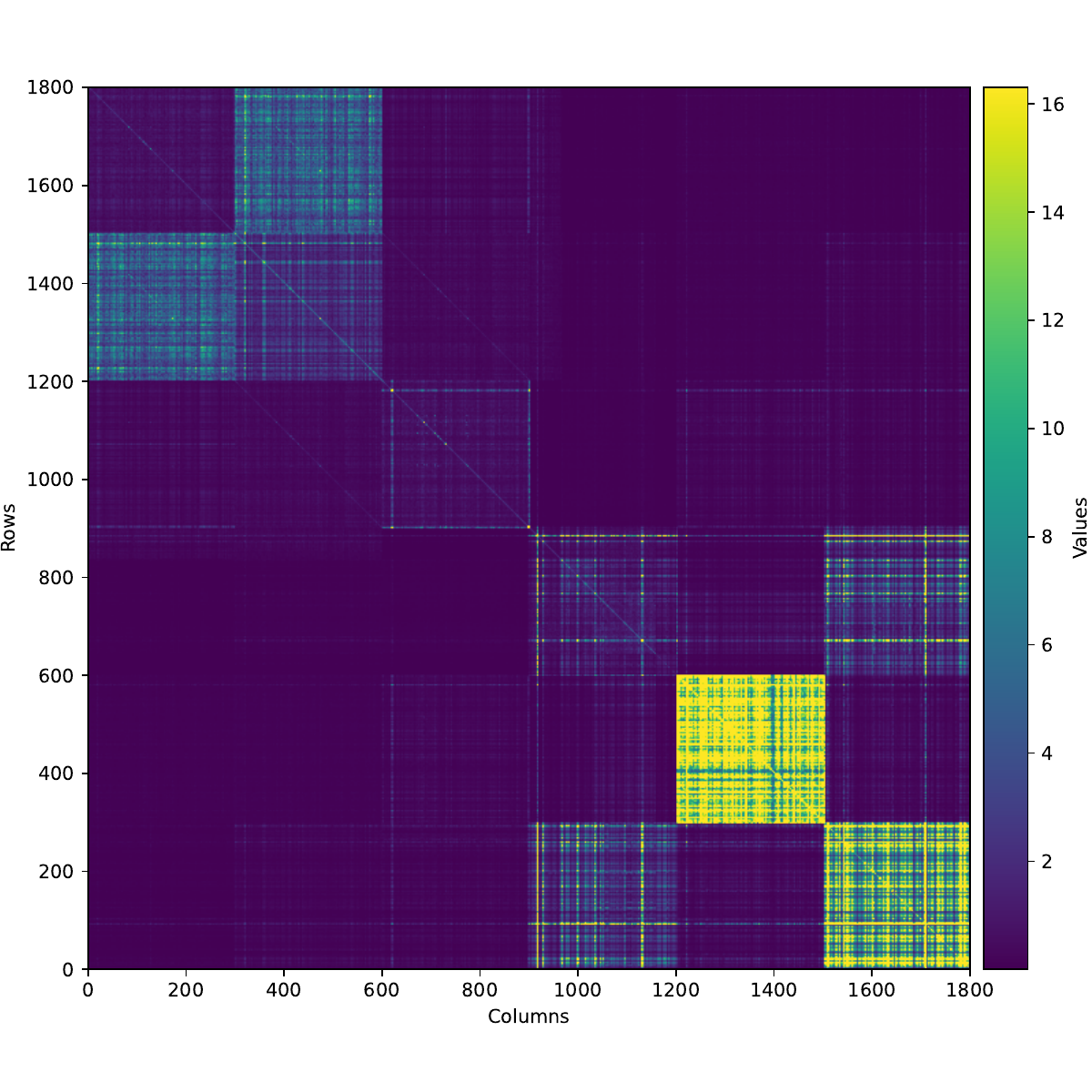}
        \captionsetup{width=.9\linewidth}
        \captionsetup{aboveskip=0pt, belowskip=12pt}
        \caption{$|\widetilde{D}^1 \nabla^2 \widetilde{\phi}^1(\widetilde{W}^1) \widetilde{D}^1|$, $t=300$ parameters.}
        \label{fig:dHd_first_layer_all_matrices}
    \end{subfigure}
    \centering
    \begin{subfigure}{0.49\linewidth}
        \includegraphics[width=0.99\linewidth]{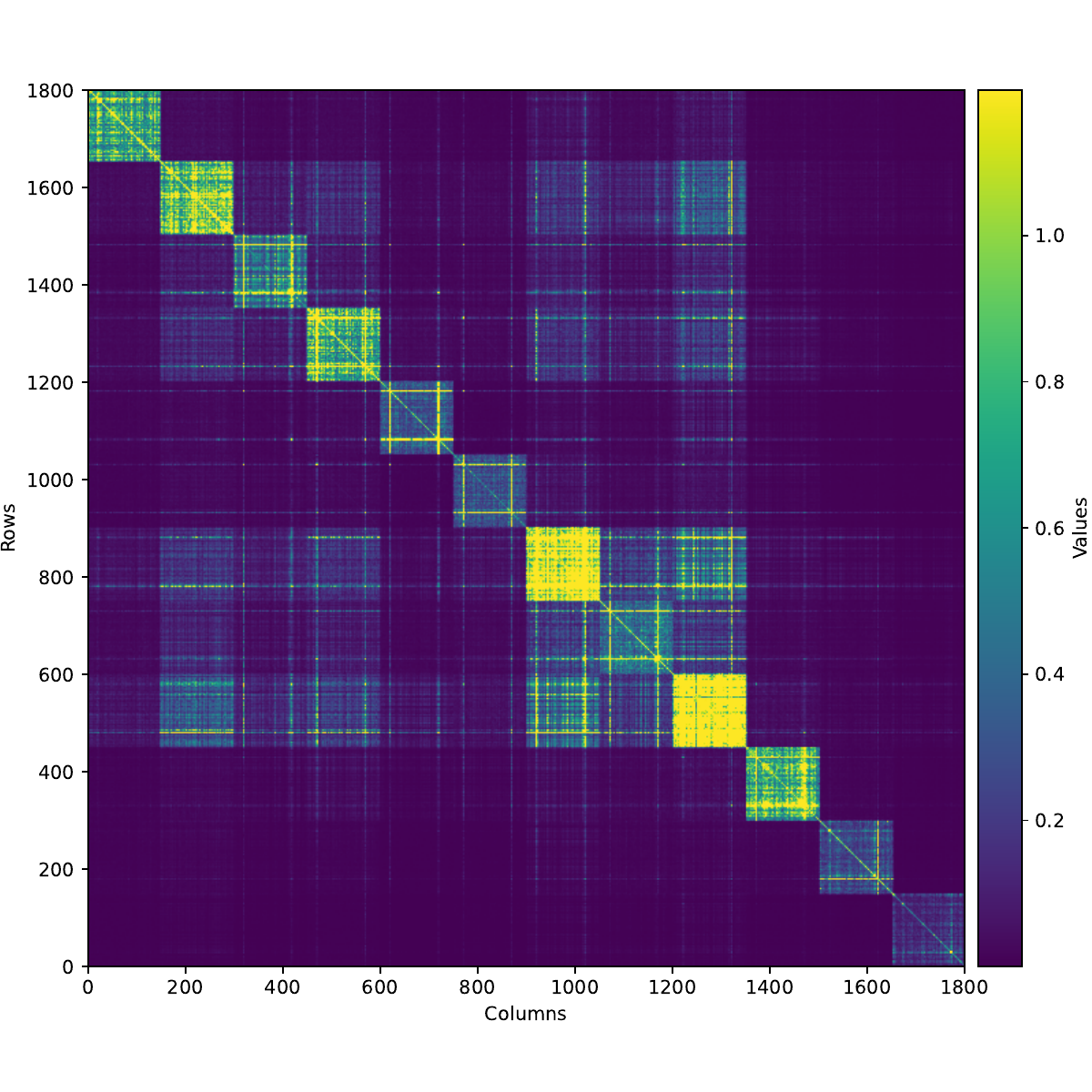}
        \captionsetup{width=.9\linewidth}
        \captionsetup{aboveskip=0pt, belowskip=12pt}
        \caption{$|\widetilde{D}_Q \nabla^2 \widetilde{\phi}_Q(\widetilde{Q}) \widetilde{D}_Q|$, $t=150$ parameters.}
        \label{fig:dHd_all_layers_first_matrix}
    \end{subfigure}
    \begin{subfigure}{0.49\linewidth}
        \includegraphics[width=0.99\linewidth]{srcs/data/dHd_all_layers_all_matrices.pdf}
        \captionsetup{width=.9\linewidth}
        \captionsetup{aboveskip=0pt, belowskip=12pt}
        \caption{$|\widetilde{D}^\star \nabla^2 \widetilde{\phi}(\widetilde{W}) \widetilde{D}^\star|$, $t=25$ parameters.}
        \label{fig:dHd_all_layers_all_matrices}
    \end{subfigure}
    \caption{Visual representation of different parts of the product $D^\star \nabla^2 \phi\left( {\cal R}(W^\star)\right) D^\star$ for different subsets of parameters. For clarity, we have plotted the absolute values of the entries in all cases to better visualize the magnitude of the elements. For all considered subsets we can clearly see a diagonal structure of the corresponding product. This justifies the Assumption~\ref{ass:reg}.}
    \label{fig:dHd}
\end{figure}

\begin{theorem}
\label{th:th_additive_prop}
$PPL'(\Bar{c}')$ defined in (\ref{eq:new_def_of_ppl}) has the additive property. In other words, the perplexity computed for the full batch $b$ will be equal to the sum of perplexities, computed for $b$ subsequent samples.
\end{theorem}
\begin{proof}
Let us define the functions $\hat{f}:\R^{l} \to \R^{l\times n}$ and $\hat{g}:\R^{l\times n} \to \R$ to be the same as $f:\R^{b\times l} \to \R^{b\times l\times n}$ and $g:\R^{b\times l\times n} \to \R^{b}$, but with fixed $b=1$, so we effectively have a reduction of one dimension.

Since the output $[f(X)]_i \in \R^{l\times n}$ for each sample $i$ from a batch is computed independently, the full output $f(X)$ can be obtained by concatenation of $b$ outputs from $\hat{f}(X_{i,:})$ for $i \in \{1, \cdots, b\}$:
$$
f(X) = \left(\hat{f}(X_{1,:}), \cdots, \hat{f}(X_{b,:})\right).
$$
The same is true for a Cross Entropy Loss function $g(f(X))$:
$$
g(f(X)) = \left(g(f(X))_1, \cdots, g(f(X))_b\right) = \left(\hat{g}(\hat{f}(X_{1,:})), \cdots, \hat{g}(\hat{f}(X_{b,:}))\right),
$$
hence the equation (\ref{eq:sum_inside_proof_additive_prop}) holds:
\begin{equation}
    \label{eq:sum_inside_proof_additive_prop}
    PPL'(X) = \sum_{i=1}^b{\left[g(f(X))\right]_i} = \sum_{i=1}^b{\hat{g}(\hat{f}(X_{i,:}))}
\end{equation}
where $X_{i,:} \in \R^l$.

In essence, equation (\ref{eq:sum_inside_proof_additive_prop}) means that the perplexity computed for the full batch $b$ will be equal to the sum of perplexities, computed for $b$ subsequent samples.

\end{proof}

\begin{corollary}
    We can compute the Hessian of the perplexity function (\ref{eq:new_def_of_ppl}) over large batch of samples by summing up several Hessians, computed on a single sample:
    \begin{equation}
        \nabla^2 PPL'(g(f(X))) = \sum_{i=1}^b{\nabla^2 PPL'(\hat{g}(\hat{f}(X_{i,:})))}.
    \end{equation}
\end{corollary}

\section{Proof of Expected Error}\label{app:mse_proof}

\begin{align*}
t_l^2 (W_l, \mathcal{G}_n^p)  &\eqdef\\
&\eqdef \Exp{ \| \widehat{W}_l - W_l \|_F^2 }/\|W_l\|_F^2 =\\
&= \frac{\sum_{i=1}^{D/g}\Exp{\|\widehat{w}_{\{i\}} - w_{\{i\}}\|_F^2}}{\sum_{i=1}^{D/g} \|w_{\{i\}}\|_F^2} =\\
&= \frac{\sum_{i=1}^{D/g}s_i^2\Exp{\|q^\dagger_{\{i\}} - w^\dagger_{\{i\}}\|_F^2}}{\sum_{i=1}^{D/g} \|w_{\{i\}}\|_F^2} \approx\\
&\approx \frac{\sum_{i=1}^{D/g}s_i^2{\rm E}_{\mathbf{w_{\{i\}}^\dagger}\sim\mathcal{N}(0, 1)}\left[\|q^\dagger_{\{i\}} - w^\dagger_{\{i\}}\|_F^2\right]}{\sum_{i=1}^{D/g} \|w_{\{i\}}\|_F^2} \eqdef\\
&\eqdef \frac{\sum_{i=1}^{D/g}s_i^2\cdot g \cdot t^2(\mathcal{G}_n^p)}{\sum_{i=1}^{D/g} \|w_{\{i\}}\|_F^2} =\\
&= t^2(\mathcal{G}_n^p)\frac{\sum_{i=1}^{D/g}\|w_{\{i\}}\|_F^2 }{\sum_{i=1}^{D/g} \|w_{\{i\}}\|_F^2} =\\
&= t^2(\mathcal{G}_n^p)
\end{align*}

That is, $t_l^2$ approximately equals the expected element-wise MSE of rounding the multivariate standard normal distribution to the grid $\mathcal{G}_n^p$

\section{Processing Hadamard Rotated Matrices}\label{app:processing_rotated}

The output $\mathbf{q}^\dagger$ and $\mathbf{s}$ of the Algorithm~\ref{alg:vqrht} represent a vector to which an RHT has been applied. Additional scale $\frac{1}{\sqrt{g}}$ ensures that this transform preserves the $L2$ norm of the vector is preserved, thus making it a random rotation, controlled by seed $\xi$, within the rotation groups of size $g$. This fact opens number of possibilities on how those quantized weights can be processed in the context of LLM inferece.

For the purposes of this section, let us consider a matrix $W \in \R^{N\times N}$ quantized with grid $\mathcal{G}_n^p$ and scales group size $g$ with Algorithm~\ref{alg:vqrht} that we want to multiply by some activation matrix $X \in \R^{K\times N}$.

\paragraph{Dequantization.}
Performing LUT restoration of $\mathbf{w}^\dagger$ from $\mathbf{q}^\dagger$ and $\mathcal{G}_n^p$, then performing a reverse Hadamard rotation in groups and unscaling with $\mathbf{s}$ aligns the representation with the original quantized matrix $W$. This is useful for validation of the representation's correctness, but the time complexity of $O(KN^2 + N^2 \log{g})$ it too slow for matrix-vector product operations ($K=1$) that are crucial for text generation inference.

\paragraph{Rotating Activations.}
The fact that RHT is a rotation, and that it's controllable by the means of seed $\xi$ allow to use it's scalar product preserving properties to compute matrix-vector or matrix-matrix multiplication fully in the rotated space. More specifically, we can apply RHT to a matrix $W$ in weight groups $W_{\{i\}}$ spanning sequential blocks of output dimension $1$ and input dimension $g$ (i.e. $W_{\{1\}} = W$\verb|[0,0:g]|, $W_{\{2\}} = W$\verb|[0,g:2g]|,...). Then, when computing the matrix-matrix product between $X$ and $W$, we can apply the same RHT (same $\xi$) to $X$ in the same groups along the input dimension and multiply the rotated matrices directly. This will result in time complexity of $O(KN^2 + KN \log{g}) = O(KN^2)$. That is, online RHT is asymptotically negligible in terms of time complexity. This approach, among other works, was used by \citet{tseng2024quipbetterllmquantization}. Table~\ref{tab:rotation_speed} showcases the effect of activations Hadamard transforms on end-to-end throughput of Llama-3.1-8B running on an RTX 4090 GPU. The observed difference doesn't exceed four percent.

\begin{table}
\centering
\caption{Throughput comparison for FLUTE kernels utilizing and omitting the Hadamard transform.}
\label{tab:rotation_speed}
\begin{tabular}{l|ccc}
\toprule
Batch Size          & wbits & \begin{tabular}[c]{@{}c@{}}Throughput with\\ Hadamard, tok/s\end{tabular} & \begin{tabular}[c]{@{}c@{}}Throughput w.o.\\ Hadamard, tok/s\end{tabular} \\ \midrule
\multirow{3}{*}{1}  & 2     & 174                                                                       & 179                                                                       \\
                    & 3     & 150                                                                       & 154                                                                       \\
                    & 4     & 139                                                                       & 143                                                                       \\ \hline
\multirow{3}{*}{4}  & 2     & 687                                                                       & 706                                                                       \\
                    & 3     & 592                                                                       & 607                                                                       \\
                    & 4     & 549                                                                       & 562                                                                       \\ \hline
\multirow{3}{*}{16} & 2     & 2433                                                                      & 2517                                                                      \\
                    & 3     & 2122                                                                      & 2205                                                                      \\
                    & 4     & 1980                                                                      & 2058                                                                                    \\
\bottomrule
\end{tabular}
\end{table}

\section{Experiment Configurations}\label{app:exp_configurations}

\subsection{Methods Specifics}

\begin{itemize}
    \item \textbf{HIGGS}: for fixed bitwidth we evaluate multiple grid dimensions $p$ by fitting grid size $n$ to match the bitwidth.
    For 3.25 bits, we use grids with $(p=2, n=88)$ and $(p=3, n=830)$.
    For 4.02 bits, we use $(p=1, n=16)$ and $(p=2, n=256)$.
    And for 4.25 bits, we use $(p=1, n=19)$ and $(p=2, n=361)$.
    We use scaling groups of size 1024 for all HIGGS experiments.

    \item \textbf{Dynamic HIGGS}: for dynamic bitwidth we only use FLUTE grids and CH8. We benchmark FLUTE kernel speed through its custom vLLM integration.
    
    \item \textbf{Normal Float (NF):} for $\approx4$ bits configurations we use the grid from \verb|bitsandbytes| library.
    We use \verb|group_size=64| (default) for 4.25 bits and  \verb|group_size=1024| (same as HIGGS) for 4.02 bits.
    For $\approx3$ bits we obtain the entropy-optimal grid and implement quantization-dequantization ourselves.
    We do not use double quantization. We benchmark kernel speed through its vLLM integration.
    
    \item \textbf{Abnormal Float (AF):} we reuse the grid generation process of~\citep{yoshida2023nf4isntinformationtheoretically} from the corresponding \href{https://github.com/davisyoshida/abnormal-floats}{repository}.
    We apply it to 3 and 4 bit grids of group size 64 and 1024, resulting in 3.02, 3.25, 4.02 and 4.25 bit configurations.
    
    \item \textbf{HQQ:} we use the \href{https://github.com/mobiusml/hqq}{official implementation of Half-Quadratic Quantization}.
    We use \verb|group_size=64| for nbits.25 bits and  \verb|group_size=1024| for nbits.02 bits, where nbits$\in\{2, 3, 4, 8\}$. We disable double quantization for fair comparison with other methods that don't use it.

    \item \textbf{GPTQ:} we use GPTQ implementation from the \href{https://github.com/AutoGPTQ/AutoGPTQ}{autogptq library}. We set \verb|bits=3,group_size=64| for 3.25 bits, \verb|bits=4,group_size=1024| for 4.02 bits and \verb|bits=4,group_size=64| for 4.25 bits. We always \verb|clip=True,mse=1| for MSE-optimal grid clipping. We benchmark MARLIN kernel speed through its vLLM integration.

    \item \textbf{AQLM:} we re-report the perplexity metric specifically for the w/o fine-tuning version of AQLM from the corresponding paper's ablation study~\citep{egiazarian2024extreme}. We benchmark kernel speed through its vLLM integration.

    \item \textbf{QuIP\#:} we re-report the perplexity metric specifically for the w/o fine-tuning version of QuIP\# from the corresponding paper's ablation study~\citep{tseng2024quipbetterllmquantization}.

    \item \textbf{QTIP\#:} we re-report the perplexity metric specifically for the w/o fine-tuning version of QTIP from the corresponding paper's ablation study~\citep{tseng2024qtip}. We re-report the metrics selectively for the 3INST grid for 2 and 4 bits and for the 1MAD grid for 3 bits. We use the end-to-end generation speed measurement scripts provided in the official repository.
 
\end{itemize}

\section{Additional Evaluations}\label{app:additional_evals}

\begin{table}
\begin{tabular}{l|l|l|llllll|l}
\toprule
Method & wbits & Wiki2 & ArcC & ArcE & PiQA & Wino & HellaS & Avg & MMLU \\
\midrule
FP16 & 16.00 & 8.644 & 31.31 & 65.53 & 74.54 & 60.54 & 47.73 & 55.93 & 32.04 \\ \hline
AF & 3.25 & 19.750 & 25.85 & 50.46 & 66.10 & 54.46 & 37.64 & 46.90 & 25.46 \\
NF & 3.25 & 17.761 & 24.15 & 49.83 & 66.70 & 53.99 & 38.57 & 46.65 & 26.70 \\
HQQ & 3.25 & 17.965 & 26.28 & 52.90 & 68.01 & 55.56 & 38.76 & 48.30 & 25.52 \\
HIGGS (p=2) & 3.25 & 13.196 & 26.96 & 56.10 & 69.97 & 54.06 & 40.79 & 49.58 & 24.62 \\
HIGGS (p=3) & 3.25 & 12.423 & 29.10 & 56.78 & 69.37 & 58.33 & 41.00 & 50.91 & 25.87 \\
HIGGS (p=4) & 3.25 & \bf{12.185} & 28.41 & 57.53 & 69.80 & 57.85 & 41.86 & \bf{51.09} & \bf{28.01} \\ \hline
HIGGS (dyn data-free) & 3.25 & 11.082 & 30.03 & 59.68 & 71.27 & 58.17 & 43.20 & 52.47 & 27.27 \\
HIGGS (dyn) & 3.25 & \bf{10.949} & 30.38 & 60.14 & 70.89 & 57.70 & 43.37 & \bf{52.49} & \bf{28.24} \\ \hline
AF & 4.02 & 10.183 & 30.29 & 61.95 & 73.39 & 58.56 & 44.59 & 53.76 & \bf{28.28} \\
NF & 4.02 & 10.703 & 28.07 & 60.90 & 73.12 & 57.77 & 44.47 & 52.87 & 25.53 \\
HQQ & 4.02 & 26.516 & 25.43 & 49.75 & 64.15 & 53.83 & 36.96 & 46.02 & 26.26 \\
HIGGS (p=1) & 4.02 & 10.167 & 31.83 & 62.58 & 72.96 & 58.88 & 45.19 & 54.29 & 26.37 \\
HIGGS (p=2) & 4.02 & 9.735 & 32.25 & 64.48 & 74.16 & 57.77 & 46.28 & \bf{54.99} & 28.12 \\
HIGGS (p=3) & 4.02 & \bf{9.641} & 29.18 & 61.36 & 72.63 & 59.59 & 45.45 & 53.64 & 26.95 \\ \hline
HIGGS (dyn data-free) & 4.00 & 9.520 & 31.40 & 63.34 & 74.27 & 60.22 & 45.79 & \bf{55.00} & 27.83 \\
HIGGS (dyn) & 4.00 & \bf{9.375} & 31.66 & 63.09 & 73.94 & 59.43 & 46.26 & 54.87 & \bf{27.87} \\ \hline
AF & 4.25 & 9.543 & 30.63 & 63.47 & 73.88 & 59.67 & 46.17 & 54.76 & 30.29 \\
NF & 4.25 & 9.575 & 30.89 & 62.37 & 74.21 & 60.54 & 45.58 & 54.72 & 28.83 \\
HQQ & 4.25 & 9.646 & 32.25 & 62.42 & 73.56 & 59.83 & 46.20 & \bf{54.85} & 29.52 \\
HIGGS (p=1) & 4.26 & 9.600 & 30.29 & 62.12 & 72.69 & 59.83 & 46.10 & 54.20 & 28.36 \\
HIGGS (p=2) & 4.26 & 9.336 & 30.89 & 64.18 & 73.56 & 59.35 & 46.18 & 54.83 & 28.79 \\
HIGGS (p=3) & 4.25 & \bf{9.299} & 30.29 & 63.43 & 73.29 & 60.77 & 46.32 & 54.82 & \bf{30.91} \\ \hline
HIGGS (dyn data-free) & 4.25 & 9.341 & 31.91 & 63.51 & 74.37 & 59.91 & 46.07 & 55.15 & \bf{29.44} \\
HIGGS (dyn) & 4.25 & \bf{9.205} & 31.57 & 63.59 & 74.27 & 59.91 & 46.68 & \bf{55.20} & 28.29 \\
\bottomrule
\end{tabular}
\caption{Quantization evaluations for Llama3.2 1b}\label{tab:llama3.2-1b}
\end{table}

\begin{table}
\begin{tabular}{l|l|l|llllll|l}
\toprule
Method & wbits & Wiki2 & ArcC & ArcE & PiQA & Wino & HellaS & Avg & MMLU \\
\midrule
FP16 & 16.00 & 6.979 & 42.15 & 74.45 & 76.71 & 69.93 & 55.29 & 63.71 & 56.07 \\ \hline
AF & 3.25 & 10.81 & 33.11 & 65.61 & 71.65 & 64.40 & 47.41 & 56.44 & 42.00 \\
NF & 3.25 & 10.04 & 37.29 & 70.37 & 73.94 & 65.35 & 49.45 & 59.28 & 44.71 \\
HQQ & 3.25 & 9.252 & 36.35 & 68.14 & 74.21 & 64.33 & 50.01 & 58.61 & 44.42 \\
HIGGS (p=2) & 3.25 & 9.156 & 35.75 & 67.89 & 74.32 & 64.96 & 50.16 & 58.61 & 46.76 \\
HIGGS (p=3) & 3.25 & 8.711 & 37.37 & 71.59 & 74.37 & 66.30 & 50.62 & 60.05 & 46.82 \\
HIGGS (p=4) & 3.25 & \bf{8.669} & 37.88 & 70.58 & 73.18 & 67.32 & 51.40 & \bf{60.07} & \bf{48.94} \\ \hline
HIGGS (dyn data-free) & 3.25 & 8.007 & 39.25 & 71.72 & 74.97 & 67.64 & 51.76 & \bf{61.07} & \bf{51.09} \\
HIGGS (dyn) & 3.25 & \bf{7.979} & 38.40 & 71.97 & 75.14 & 67.25 & 52.42 & 61.03 & 50.59 \\ \hline
AF & 4.02 & 7.697 & 41.13 & 73.36 & 76.39 & 68.67 & 53.81 & 62.67 & 53.08 \\
NF & 4.02 & 7.819 & 39.68 & 73.57 & 76.01 & 67.56 & 53.17 & 62.00 & 51.13 \\
HQQ & 4.02 & 12.57 & 34.73 & 68.77 & 72.74 & 64.09 & 47.66 & 57.60 & 36.89 \\
HIGGS (p=1) & 4.02 & 7.695 & 41.30 & 73.99 & 75.95 & 67.80 & 53.35 & 62.48 & 53.01 \\
HIGGS (p=2) & 4.02 & 7.507 & 41.64 & 73.86 & 76.33 & 68.27 & 54.11 & \bf{62.84} & 53.82 \\
HIGGS (p=3) & 4.02 & \bf{7.464} & 40.44 & 71.76 & 76.66 & 68.82 & 53.87 & 62.31 & \bf{54.20} \\ \hline
HIGGS (dyn data-free) & 4.00 & 7.399 & 39.76 & 73.19 & 76.22 & 67.96 & 53.70 & 62.17 & 54.37 \\
HIGGS (dyn) & 4.00 & \bf{7.295} & 41.04 & 73.27 & 76.44 & 68.51 & 53.75 & \bf{62.60} & \bf{54.43} \\ \hline
AF & 4.25 & 7.365 & 40.61 & 73.15 & 76.88 & 68.27 & 54.43 & 62.67 & 54.53 \\
NF & 4.25 & 7.395 & 41.98 & 73.86 & 76.71 & 68.75 & 54.35 & \bf{63.13} & 54.51 \\
HQQ & 4.25 & 7.351 & 42.75 & 72.77 & 76.93 & 68.82 & 53.89 & 63.03 & 53.33 \\
HIGGS (p=1) & 4.26 & 7.459 & 40.19 & 72.31 & 76.12 & 67.96 & 53.97 & 62.11 & 53.43 \\
HIGGS (p=2) & 4.26 & 7.339 & 38.23 & 69.82 & 76.22 & 68.59 & 54.23 & 61.42 & 54.44 \\
HIGGS (p=3) & 4.25 & \bf{7.306} & 40.53 & 73.65 & 76.39 & 69.06 & 53.95 & 62.72 & \bf{54.54} \\ \hline
HIGGS (dyn data-free) & 4.25 & 7.266 & 40.02 & 73.44 & 76.01 & 69.53 & 54.01 & 62.60 & \bf{54.66} \\
HIGGS (dyn) & 4.25 & \bf{7.216} & 41.21 & 73.40 & 76.17 & 69.06 & 54.22 & \bf{62.81} & 54.63 \\
\bottomrule
\end{tabular}
\caption{Quantization evaluations for Llama3.2 3b}\label{tab:llama3.2-3b}
\end{table}

\begin{table}

\begin{tabular}{l|l|l|llllll|l}
\toprule
Method & wbits & Wiki2 & ArcC & ArcE & PiQA & Wino & HellaS & Avg & MMLU \\
\midrule
FP16 & 16.00 & 6.497 & 51.71 & 81.78 & 79.92 & 73.80 & 59.12 & 69.26 & 68.20 \\ \hline
AF & 3.25 & 8.848 & 43.69 & 74.92 & 77.75 & 69.61 & 53.52 & 63.90 & 57.51 \\
NF & 3.25 & 8.663 & 43.00 & 75.51 & 77.48 & 71.11 & 54.73 & 64.37 & 58.42 \\
HQQ & 3.25 & 8.273 & 46.59 & 77.23 & 77.97 & 71.19 & 55.70 & 65.73 & 59.41 \\
HIGGS (p=2) & 3.25 & 7.765 & 44.80 & 75.97 & 77.53 & 72.69 & 56.58 & 65.51 & 61.32 \\
HIGGS (p=3) & 3.25 & 7.659 & 49.74 & 80.35 & 77.64 & 71.67 & 56.59 & 67.20 & 62.26 \\
HIGGS (p=4) & 3.25 & \bf{7.469} & 48.72 & 78.66 & 79.65 & 72.77 & 56.31 & \bf{67.22} & \bf{63.83} \\ \hline
HIGGS (dyn data-free) & 3.25 & 7.351 & 48.63 & 78.49 & 78.51 & 70.48 & 56.76 & 66.58 & 63.79 \\
HIGGS (dyn) & 3.25 & \bf{7.204} & 47.78 & 79.12 & 78.89 & 72.53 & 57.27 & \bf{67.12} & \bf{64.36} \\ \hline
AF & 4.02 & 7.107 & 50.51 & 79.84 & 79.43 & 73.64 & 57.69 & 68.22 & 65.28 \\
NF & 4.02 & 7.084 & 49.74 & 80.35 & 79.43 & 73.95 & 58.14 & 68.32 & 65.59 \\
HQQ & 4.02 & 9.393 & 48.29 & 77.48 & 77.04 & 71.59 & 55.50 & 65.98 & 60.21 \\
HIGGS (p=1) & 4.02 & 7.013 & 49.15 & 80.18 & 79.65 & 72.53 & 57.70 & 67.84 & 65.64 \\
HIGGS (p=2) & 4.02 & 6.901 & 49.15 & 81.27 & 79.87 & 73.24 & 58.20 & 68.35 & \bf{65.98} \\
HIGGS (p=3) & 4.02 & \bf{6.833} & 50.17 & 81.82 & 79.98 & 72.77 & 58.59 & \bf{68.67} & 65.77 \\ \hline
HIGGS (dyn data-free) & 4.00 & 6.835 & 50.68 & 80.60 & 79.60 & 73.40 & 58.17 & 68.49 & 66.05 \\
HIGGS (dyn) & 4.00 & \bf{6.720} & 50.68 & 81.48 & 79.87 & 73.40 & 58.55 & \bf{68.80} & \bf{66.41} \\ \hline
AF & 4.25 & 6.758 & 51.96 & 80.68 & 79.16 & 72.85 & 58.57 & 68.65 & 66.41 \\
NF & 4.25 & 6.882 & 51.62 & 80.98 & 79.27 & 73.72 & 58.41 & 68.80 & \bf{67.21} \\
HQQ & 4.25 & 6.777 & 50.51 & 81.02 & 79.33 & 73.80 & 58.37 & 68.61 & 66.44 \\
HIGGS (p=1) & 4.26 & 6.838 & 50.51 & 80.35 & 79.71 & 72.85 & 58.47 & 68.38 & 66.40 \\
HIGGS (p=2) & 4.26 & \bf{6.736} & 53.75 & 81.82 & 80.09 & 74.03 & 58.24 & \bf{69.59} & 67.14 \\
HIGGS (p=3) & 4.25 & 6.741 & 51.37 & 81.10 & 79.71 & 73.32 & 58.79 & 68.86 & 66.97 \\ \hline
HIGGS (dyn data-free) & 4.25 & 6.783 & 50.94 & 81.31 & 80.09 & 74.03 & 58.33 & 68.94 & \bf{66.86} \\
HIGGS (dyn) & 4.25 & \bf{6.664} & 51.28 & 81.61 & 79.92 & 74.59 & 58.81 & \bf{69.24} & 66.32 \\
\bottomrule
\end{tabular}

\caption{Quantization evaluations for Llama3.1 8b Instruct}\label{tab:llama3.1-8b-Instruct}
\end{table}

\begin{table}
\begin{tabular}{l|l|l|llllll|l}
\toprule
Method & wbits & Wiki2 & ArcC & ArcE & PiQA & Wino & HellaS & Avg & MMLU \\
\midrule
FP16 & 16.00 & 2.541 & 60.67 & 87.25 & 83.13 & 79.64 & 66.48 & 75.43 & 78.51 \\ \hline
AF & 3.25 & 102350 & 51.54 & 81.52 & 80.85 & 75.14 & 62.54 & 70.32 & 60.53 \\
NF & 3.25 & 41.76 & 54.78 & 81.57 & 80.85 & 76.40 & 63.22 & 71.36 & 63.42 \\
HQQ & 3.25 & 4.057 & 57.17 & 84.30 & 81.23 & 77.82 & 64.21 & \bf{72.95} & 75.32 \\
HIGGS (p=2) & 3.25 & 4.297 & 54.44 & 84.09 & 81.34 & 70.88 & 63.76 & 70.90 & 73.44 \\
HIGGS (p=3) & 3.25 & 4.023 & 57.08 & 84.55 & 81.56 & 67.48 & 64.14 & 70.96 & 75.10 \\
HIGGS (p=4) & 3.25 & \bf{3.792} & 55.72 & 83.96 & 82.15 & 72.06 & 65.34 & 71.85 & \bf{75.77} \\ \hline
HIGGS (dyn data-free) & 3.25 & 3.675 & 58.45 & 85.48 & 81.61 & 79.40 & 64.39 & 73.87 & 76.36 \\
HIGGS (dyn) & 3.25 & \bf{3.466} & 57.68 & 85.52 & 82.15 & 78.93 & 65.23 & \bf{73.90} & \bf{76.98} \\ \hline
AF & 4.02 & 513.5 & 21.16 & 25.46 & 53.16 & 59.83 & 52.34 & 42.39 & 25.71 \\
NF & 4.02 & 2084 & 20.73 & 25.51 & 52.67 & 52.17 & 27.34 & 35.68 & 23.47 \\
HQQ & 4.02 & 4.023 & 55.63 & 83.21 & 80.85 & 74.51 & 62.21 & 71.28 & 74.67 \\
HIGGS (p=1) & 4.02 & 3.115 & 58.70 & 85.35 & 82.54 & 77.82 & 65.83 & 74.05 & 77.27 \\
HIGGS (p=2) & 4.02 & 2.986 & 60.15 & 86.28 & 82.64 & 78.85 & 66.20 & \bf{74.83} & 77.59 \\
HIGGS (p=3) & 4.02 & \bf{2.956} & 59.64 & 86.07 & 82.97 & 77.90 & 65.98 & 74.51 & \bf{77.80} \\ \hline
HIGGS (dyn data-free) & 4.00 & 3.133 & 60.15 & 86.15 & 82.75 & 78.61 & 66.04 & 74.74 & 77.68 \\
HIGGS (dyn) & 4.00 & \bf{2.827} & 60.24 & 86.74 & 82.70 & 79.64 & 66.17 & \bf{75.10} & \bf{78.24} \\ \hline
AF & 4.25 & 2.982 & 58.96 & 86.32 & 83.03 & 80.27 & 65.89 & \bf{74.89} & 77.11 \\
NF & 4.25 & 3.065 & 57.25 & 86.28 & 82.70 & 78.14 & 66.04 & 74.08 & 78.12 \\
HQQ & 4.25 & 2.873 & 60.67 & 87.04 & 82.75 & 77.98 & 65.75 & 74.84 & \bf{78.47} \\
HIGGS (p=1) & 4.26 & 2.935 & 58.53 & 86.07 & 82.32 & 77.27 & 65.83 & 74.00 & 77.62 \\
HIGGS (p=2) & 4.26 & 2.852 & 59.30 & 86.62 & 82.21 & 78.85 & 66.12 & 74.62 & 78.00 \\
HIGGS (p=3) & 4.25 & \bf{2.834} & 60.49 & 86.36 & 82.59 & 78.53 & 66.02 & 74.80 & 77.70 \\ \hline
HIGGS (dyn data-free) & 4.25 & 2.956 & 60.67 & 86.74 & 82.81 & 78.77 & 66.05 & 75.01 & \bf{78.14} \\
HIGGS (dyn) & 4.25 & \bf{2.787} & 60.92 & 86.32 & 83.03 & 79.87 & 66.35 & \bf{75.30} & 78.09 \\
\bottomrule
\end{tabular}
\caption{Quantization evaluations for Llama3.1 70b. We don't quantize attention's value layer for the first transformer block of the model in non-dynamic setups.}\label{tab:llama3.1-70b}
\end{table}

\begin{table}

\begin{tabular}{l|l|l|llllll|l}
\toprule
Method & wbits & Wiki2 & ArcC & ArcE & PiQA & Wino & HellaS & Avg & MMLU \\
\midrule
FP16 & 16.00 & 6.128 & 47.70 & 80.51 & 78.67 & 72.93 & 60.02 & 67.97 & 74.13 \\ \hline
AF & 3.25 & 7.542 & 46.59 & 79.38 & 77.37 & 68.19 & 55.27 & 65.36 & 67.16 \\
NF & 3.25 & 7.588 & 44.28 & 76.94 & 77.42 & 67.72 & 55.34 & 64.34 & 68.27 \\
HQQ & 3.25 & 7.247 & 46.84 & 76.01 & 78.13 & 67.25 & 56.68 & 64.98 & 69.58 \\
HIGGS (p=2) & 3.25 & 6.808 & 46.59 & 77.78 & 78.78 & 71.19 & 57.97 & 66.46 & 70.90 \\
HIGGS (p=3) & 3.25 & 6.732 & 46.25 & 77.78 & 78.73 & 70.72 & 57.97 & 66.29 & 71.41 \\
HIGGS (p=4) & 3.25 & \bf{6.669} & 48.04 & 79.55 & 78.07 & 71.59 & 58.08 & \bf{67.06} & \bf{71.73} \\ \hline
HIGGS (dyn data-free) & 3.25 & 6.604 & 47.61 & 79.17 & 78.40 & 72.30 & 57.56 & 67.01 & 71.26 \\
HIGGS (dyn) & 3.25 & \bf{6.566} & 48.21 & 79.34 & 78.62 & 71.27 & 57.64 & \bf{67.01} & \bf{71.41} \\ \hline
AF & 4.02 & 6.538 & 45.31 & 77.86 & 77.91 & 71.82 & 58.28 & 66.24 & 72.23 \\
NF & 4.02 & 6.538 & 47.61 & 79.59 & 78.35 & 69.85 & 58.98 & 66.88 & 72.35 \\
HQQ & 4.02 & 7.767 & 44.88 & 76.22 & 77.26 & 68.98 & 55.33 & 64.53 & 70.72 \\
HIGGS (p=1) & 4.02 & 6.407 & 47.35 & 79.71 & 78.51 & 70.40 & 58.60 & 66.92 & 72.78 \\
HIGGS (p=2) & 4.02 & 6.359 & 48.81 & 80.30 & 79.00 & 72.22 & 58.68 & \bf{67.80} & 72.97 \\
HIGGS (p=3) & 4.02 & \bf{6.325} & 47.27 & 79.97 & 78.67 & 72.14 & 59.10 & 67.43 & \bf{73.39} \\ \hline
HIGGS (dyn data-free) & 4.00 & 6.334 & 48.89 & 79.67 & 78.78 & 70.72 & 58.87 & \bf{67.39} & 72.58 \\
HIGGS (dyn) & 4.00 & \bf{6.291} & 47.70 & 78.75 & 78.84 & 71.67 & 59.04 & 67.20 & \bf{73.10} \\ \hline
AF & 4.25 & 6.352 & 48.29 & 79.46 & 78.18 & 71.27 & 58.83 & 67.21 & 73.29 \\
NF & 4.25 & 6.340 & 46.67 & 79.63 & 78.35 & 71.19 & 59.00 & 66.97 & 73.44 \\
HQQ & 4.25 & 6.358 & 46.93 & 79.97 & 78.89 & 71.11 & 59.38 & 67.26 & 73.32 \\
HIGGS (p=1) & 4.26 & 6.343 & 48.81 & 79.08 & 78.56 & 73.01 & 58.64 & 67.62 & 73.67 \\
HIGGS (p=2) & 4.26 & 6.298 & 46.76 & 79.17 & 78.84 & 73.16 & 59.05 & 67.40 & 73.30 \\
HIGGS (p=3) & 4.25 & \bf{6.278} & 47.70 & 79.92 & 78.13 & 73.16 & 59.20 & \bf{67.62} & \bf{73.86} \\ \hline
HIGGS (dyn data-free) & 4.25 & 6.282 & 48.21 & 79.71 & 78.78 & 70.56 & 58.85 & 67.22 & \bf{73.32} \\
HIGGS (dyn) & 4.25 & \bf{6.255} & 48.72 & 80.22 & 78.84 & 70.88 & 59.01 & \bf{67.53} & 73.28 \\
\bottomrule
\end{tabular}

\caption{Quantization evaluations for Qwen2.5 7B}\label{tab:qwen2.5-7b}
\end{table}

\end{document}